\documentclass[]{article}
\usepackage[margin=1in]{geometry}
\usepackage{amsmath}
\usepackage{amsfonts}
\usepackage{amssymb}
\usepackage{amsthm}
\usepackage{graphicx}
\usepackage{fourier}
\usepackage{natbib}
\usepackage{apalike}

\newcommand{\uni}{\cup} 

\newcommand{\Prp}{\ensuremath{\mathcal{P}}}

\newtheorem{definition}{Definition}[section]

\newtheorem{lemma}{Lemma}[section]
 
\newtheorem{proposition}{Proposition}[section]
\newtheorem{corollary}{Corollary}[section]

\newtheorem{example}{Example}[section]

\usepackage{times}

\title{Imagining Probabilistic Belief Change as Imaging\\
(Technical Report)}

\author{Gavin Rens and Thomas Meyer\\
University of Cape Town, South Africa\\
\{grens,tmeyer\}@cs.uct.ac.za}

\begin{document}

\maketitle

\begin{abstract}
Imaging is a form of probabilistic belief change which could be employed for both revision and update. In this paper, we propose a new framework for probabilistic belief change based on imaging, called Expected Distance Imaging (EDI). EDI is sufficiently general to define Bayesian conditioning and other forms of imaging previously defined in the literature. We argue that, and investigate how, EDI can be used for both revision and update. EDI's definition depends crucially on a weight function whose properties are studied and whose effect on belief change operations is analysed. Finally, four EDI instantiations are proposed, two for revision and two for update, and probabilistic rationality postulates are suggested for their analysis. 
\end{abstract}

\section{Introduction}

From the perspective of classical (Boolean) belief change, the work of \citet{agm85} is regarded as  the foundation theory for belief revision (AGM theory). Typically, belief change (in a static world) can be categorized as expansion, revision or contraction, and is performed on a belief set, the set of sentences $K$ closed under logical consequence. Revision is when new information $\alpha$ is (possibly) inconsistent with $K$ and $K$ is (minimally) modified so that the new $K$ remains consistent and entails $\alpha$. Revision is the process which takes place when an agent modifies its beliefs due to receiving new information not previously known or which is more relevant or trustworthy. Except for the movement of information, the physical world is assumed to be completely unchanging.

Whereas belief revision is considered to take place in a static environment, belief update is thought to be the change in beliefs which takes place due to a dynamic environment. Update refers to the process of bringing beliefs up to date precisely because the world has changed and the agent needs a new, `matching' view on the world.

From the perspective of classical (Boolean) belief change, \citet{km92} developed the first serious theory of update (KM theory). Their theory is different from that of AGM in that their rationality postulates are derived from the understanding that update occurs in a dynamic environment.

However, simply applying AGM theory for revision and KM theory for update has been (indirectly) challenged \citep{fh99a,k01a,n11,l07a}. Further, how to categorize a belief change operator is more challenging when notions of uncertainty are considered, for instance, when using probabilities and rankings. The very definition of belief revision and belief update become more problematic under notions of uncertainty.

One kind of probabilistic belief change operation which could potentially `relax' the tension between revision and update is \textit{imaging}. 
David Lewis (1976) first proposed \textit{imaging} to analyse conditional reasoning in probabilistic settings, and it has recently been the focus of several works on probabilistic belief change \citep{rno10,cnss14,rmc16b}. Imaging is the approach of moving the belief in worlds possible at one moment to similar worlds compatible with evidence received at a next moment. In other words, the `belief-mass' is shifted to the `images' of the worlds currently believed possible, where the images are the worlds related via new evidence to the currently believed worlds.

One of the main benefits of imaging is that it overcomes the problem with Bayesian conditioning, namely, being undefined when evidence is inconsistent with current beliefs.
\citet{g88} and \citet{rmc16b} proposed generalizations of Lewis's original definition. In this paper we propose a new generalization of imaging -- or a family of imaging-based belief change operators -- and analyse other probabilistic belief change methods with respect to it. In particular, we propose a version of imaging based on the movement of probability mass weighted by the inverse of distances between possible worlds.

Whether imaging is applicable to revision or update (or both) is still an open question.
\citet{g88} says "...the imaging method is a general revision method because it gives nontrivial results when...the new information to be accommodated...[contradicts the current beliefs]".
\citet{cnss14} explore the use of Lewis imaging as a means to construct probabilistic belief revision. They present explicit constructions of three candidates strategies based on imaging and
investigate their properties.
\citep{rmc16b} define an imaging operation which relaxes the unique-closest-world assumption of Lewis imaging, and they provide a method of revising (via imaging) a potentially infinite set of belief states in a finite procedure.
On the other hand, some researchers have considered imaging to be the probabilistic version of update \citep{km92,dp93,n11}.
\citet{rno10} propose a version of imaging for probabilistic belief erasure.
In fact, \citet{l76} himself never said that imaging was meant to be interpreted as one or the other.
In this paper we continue to investigate the relationship of imaging to revision and update.

The paper makes four contributions. 1) We define Expected Distance Imaging (EDI) and show how Bayesian conditioning, Lewis imaging and generalized imaging can be couched as EDI operations. 2) We define four (new) instantiations of the EDI operation. 3) We define a weight function (as used in EDI), and several properties such a function might have, and explore which of these properties are satisfied by different instantiations of weight functions (for different versions of the EDI operator). 4) We propose a set of rationality postulates for probabilistic belief update.

Due to space limitations, all proofs of propositions have been omitted. Please refer to the appendix for the proofs.

\section{Background}

We shall work with a finitely generated classical propositional logic. Let $\Prp$ be a finite set of $n$ \textit{atoms}. Formally, a \textit{world} $w$ is a unique assignment of truth values to all the atoms in $\Prp$.
There are thus $2^n$ conceivable worlds. 
An agent may consider some non-empty subset $W$ of the conceivable worlds; $W$ is called the possible worlds.
The classical notion of satisfaction is used. World $w$ satisfies (is a model of) $\alpha$ is written $w\Vdash\alpha$.
Let $L$ be all propositional formulae which can be formed from $\Prp$ and the logical connectives $\land$ and $\lnot$, with $\top$ abbreviating tautology and $\bot$ abbreviating contradiction.
Let $\beta$ be a sentence in $L$. $\mathit{Mod}(\beta)$ denotes the set of models of $\beta$.
$\beta$ entails $\alpha$ (denoted $\beta\models\alpha$) iff $\mathit{Mod}(\beta)\subseteq\mathit{Mod}(\alpha)$.
$\beta$ is equivalent to $\alpha$ (denoted $\beta\equiv\alpha$) iff $\mathit{Mod}(\beta)=\mathit{Mod}(\alpha)$.

Often, in the exposition of this paper, a world will be referred to by its truth vector. For instance, if a two-atom vocabulary is placed in order $\langle q,r\rangle$ and $w_3\Vdash \lnot q\land r$, then $w_3$ may be referred to as $01$.

In this work, the basic semantic element of an agent's beliefs is a probability distribution or a \textit{belief state}
\[b=\{(w_1,p_1), (w_2,p_2), \ldots, (w_n,p_n)\},\] where $\{w_1, w_2, \ldots, w_n\}=W$ and $p_i$ is the probability (real number) that $w_i$ is the actual world in which the agent is, with $\sum_{(w,p)\in b}p=1$.
For parsimony, let $b=\langle p_1,\ldots, p_n \rangle$ be the probabilities that belief state $b$ assigns to $w_1,\ldots, w_n$ where, for instance, $\langle w_1,w_2,w_3,w_4\rangle$ $=$ $\langle 11,10,01,00\rangle$, and $\langle w_1,w_2,\ldots,w_8\rangle$ $=$ $\langle 111,110,\ldots,000\rangle$.
$b(\alpha)$ abbreviates $\sum_{w\in\mathit{Mod}(\alpha)}b(w)$.

It is not yet universally agreed what belief change means in a probabilistic setting. One school of thought says that probabilistic expansion (restricted revision) is equivalent to Bayesian conditioning (\citet[Chap.~5]{g88} and \citet{v99b} mention this, but no not necessarily agree with it). This is evidenced by Bayesian conditioning ($\mathsf{BC}$) being defined only when $b(\alpha)\neq0$, thus making $\mathsf{BC}$ expansion equivalent to $\mathsf{BC}$ revision.
In other words,
one could define expansion to be 
\[
b\:\mathsf{BC}\:\alpha =\{(w,p)\mid w\in W, p= b( w\mid\alpha),b(\alpha) \neq 0\},
\]
where $b( w\mid\alpha)$ can be defined as $b(\phi_w\land\alpha)/b(\alpha)$ and $\phi_w$ is a sentence identifying $w$ (i.e., a complete theory for $w$).
Note that $b\:\mathsf{BC}\:\alpha = \emptyset$ iff $b(\alpha) = 0$. This implies that $\mathsf{BC}$ is ill-defined when $b(\alpha) = 0$.


We may write $b*\alpha$ as $b^*_\alpha$ so that we can write $b^*_\alpha(w)$, where $*$ is an arbitrary belief change operator.


The technique of \textit{Lewis imaging} for the revision of belief states requires a notion of similarity between worlds. In fact, he implicitly assumes the availability of a mapping of worlds to a total order $\leq_w$ on worlds for every fixed $w\in W$. Let $w^\alpha$ be the least $\alpha$-world with respect to $\leq_w$. That is, $w^\alpha \Vdash\alpha$ and: if $w \not\Vdash\alpha$, then $w^\alpha\leq_w w''$ for all $w''\in \mathit{Mod}(\alpha)$, and if $w\Vdash\alpha$, then $w^\alpha = w$.

If we indicate Lewis's original imaging operation with $\mathsf{LI}$, then his definition can be stated as
\[
b\,\mathsf{LI}\,\alpha := \{(w,p)\mid w\in W, p=0 \mbox{ if } w\not\Vdash\alpha,\mbox{ else }p=\sum_{w'\in W,w'^\alpha=w}b(w')\}.
\]
He calls $b\,\mathsf{LI}\,\alpha$ the image of $b$ on $\alpha$.
In words, $b\,\mathsf{LI}\,\alpha(w)$ is zero if $w$ does not model $\alpha$, but if it does, then $w$ retains all the probability it had and accrues the probability mass from all the non-$\alpha$-worlds closest to it.
This form of imaging only shifts probabilities around; no probabilities are magnified or shrunk. The probabilities in $b\,\mathsf{LI}\,\alpha$ thus sum to 1 without the need for any normalization.

\section{Non-unique-closest-world Approaches}

Every world having a unique closest $\alpha$-world is quite a strong requirement. We now mention an approach which relaxes the uniqueness requirement.

\citet{g88} describes his generalization of Lewis imaging (which he calls \textit{general imaging}) as ``... instead of moving all the probability assigned to a world  $W^i$ by a probability function $P$ to a unique (``closest'') $A$-world $W^j$, when imaging on $A$, one can introduce the weaker requirement that the probability of $W^i$ be distributed among \textit{several} $A$-worlds (that are ``equally close'').''
If we interpret G\"ardenfors' approach correctly, he does not provide a constructive method but insists that $b^\#_\alpha(\alpha)=1$, where $b^\#_\alpha$ is the image (change) of $b$ on $\alpha$.

\citet{rmc16b} introduced \textit{generalized} imaging via a constructive method. It is a particular instance of G\"ardenfors' general imaging.
They use a pseudo-distance measure between worlds, as defined by \citet{lms01} and adopted by \citet{cnss14}.
\begin{definition}
A pseudo-distance function $d~:~W \times W \to \mathbb{Z}$ satisfies the following four conditions: for all worlds $w, w', w'' \in W$,
\begin{enumerate}
\itemsep=0pt
\item $d(w, w') \geq 0$ (Non-negativity)
\item $d(w, w) = 0$ (Identity)
\item $d(w, w') = d(w', w)$ (Symmetry)
\item $d(w, w') + d(w', w'') \geq d(w, w'')$ (Triangular\\ Inequality)
\end{enumerate}
\label{def:pseudo-func}
\end{definition}

Note that $d$ induces a mapping from worlds to total preorders $\leq^d_w$ on worlds as follows. $w'\leq^d_w w''$ iff $d(w,w')\leq d(w,w'')$.
Note that conditions 2 and 4 make $\leq^d_w$ a total preorder, and that adding conditions 1 and 3 do not necessarily make $\leq^d_w$ a (total) partial order. This makes it possible for a world to have more than one closest worlds.
One may also want to impose a condition on a distance function such that any two distinct worlds must have some distance between them: For all $w,w'\in W$, if $w\neq w'$, then $d(w, w') > 0$ (Faithfulness).

Let $\mathit{Min}(\alpha,w,d)$ be the set of $\alpha$-worlds closest to $w$ with respect to pseudo-distance $d$. Formally,
\[\mathit{Min}(\alpha,w,d):=\{w'\Vdash\alpha\mid \forall w''\Vdash\alpha, d(w',w)\leq d(w'',w)\},\]
where $d(\cdot)$ is some pseudo-distance measure between worlds (e.g., Hamming or Dalal distance).
For instance, using Hamming distance for $d$ and vocabulary $\{q,r\}$, 
\begin{itemize}
\item $\mathit{Min}(\lnot q,11,d)=\{01\}$
\item $\mathit{Min}(\lnot q,10,d)=\{00\}$
\item $\mathit{Min}(\lnot q,01,d)=\{01\}$
\item $\mathit{Min}(\lnot q,00,d)=\{00\}$
\end{itemize}

\emph{Generalized imaging} (denoted $\mathsf{GI}$) is defined as
\[
b\:\mathsf{GI}\:\alpha :=\big\{(w,p)\mid w\in W, p=0 \mbox{ if } w\not\Vdash\alpha,\mbox{else }p=\sum_{\substack{w'\in W\\w\in\mathit{Min}(\alpha,w',d)}}b(w')/|\mathit{Min}(\alpha,w',d)|\big\}.
\]

$b\:\mathsf{GI}\:\alpha$ is the new belief state produced by taking the generalized image of $b$ on $\alpha$.
In words, the probability mass of non-$\alpha$-worlds is shifted to their closest $\alpha$-worlds, such that if a non-$\alpha$-world $w^\times$ with probability $p$ has $n$ closest $\alpha$-worlds (equally distant), then each of these closest $\alpha$-worlds gets $p/n$ mass from $w^\times$.

As an example, let belief state $b_{1.0}:=\langle 1,0,0,0\rangle$ and let belief state $b_{.3/.7}:=\langle 0.3,0.7,0,0\rangle$ for vocabulary $\{q,r\}$. Then
\begin{itemize}
\itemsep=0pt
\item $b_{0.3/0.7}\:\mathsf{GI}\:\lnot q(11)=b_{0.3/0.7}\:\mathsf{GI}\:\lnot q(10)=0$
\item $b_{0.3/0.7}\:\mathsf{GI}\:\lnot q(01)=0.3/1 + 0/1=0.3$
\item $b_{0.3/0.7}\:\mathsf{GI}\:\lnot q(00)=0.7/1 + 0/1=0.7$
\end{itemize}
and
\begin{itemize}
\itemsep=0pt
\item $b_{1.0}\:\mathsf{GI}\:\lnot q(11)=b_{1.0}\:\mathsf{GI}\:\lnot q(10)=0$
\item $b_{1.0}\:\mathsf{GI}\:\lnot q(01)=1/1 + 0/1=1$
\item $b_{1.0}\:\mathsf{GI}\:\lnot q(00)=0/1 + 0/1=0$
\end{itemize}

\section{Expected Distance Imaging}

In this section, we define \textit{expected distance imaging} (EDI) and some properties of interest, then we define two instantiations of EDI and define Bayesian conditioning, Lewis imaging and generalized imaging in terms of EDI. Each operator is considered with respect to the properties. We end with a discussion about the five operations.

\subsection{Definition and Properties}

\citet{rm15b} proposed determining the new probability of an $\alpha$-world $w^\alpha$ as the weighted average of the current probabilities of all worlds $w'$, where the weights are (inversely) proportional to the distance between $w^\alpha$ and the $w'$. The reason for bringing in weights is that the probability of worlds $w'$ less (more) similar to the $\alpha$-world $w^\alpha$ under consideration, should contribute less (more) to the new probability mass of $w^\alpha$. (Distance and similarity are taken to be inversely proportional in this context.)
Building on that idea, we introduce an imaging framework based on weighted distances.
The weight functions are defined in terms of inverse-distance functions, which are defined in terms of the pseudo-distance function.

First, we introduce \textit{potentially inverse-distance} functions $\iota_d: W\times W\to \mathbb{R}$ and propose considering the following seven postulates for such functions.
For all $w,w',w'',w'''\in W$ and all pseudo-distance functions $d$,
\begin{enumerate}
\itemsep=0pt
\item $\iota_d(w,w')\geq0$ (Non-negativity)
\item $\iota_d(w,w)=1$ (Identity)
\item $\iota_d(w,w')=\iota_d(w',w)$ (Symmetry)
\item if $d(w,w')\geq d(w'',w''')$, then $\iota_d(w,w')\leq\iota_d(w'',w''')$ (Weak Inversity)
\item if $d(w,w')>d(w'',w''')$, then $\iota_d(w,w')<\iota_d(w'',w''')$ (Strict Inversity)
\item if $d(w,w')=d(w'',w''')$, then $\iota_d(w,w')=\iota_d(w'',w''')$ (Equi-distance).
\item if $w\neq w'$, then $\iota_d(w,w')<1$ (Faithfulness)
\end{enumerate}
\begin{definition}
A \emph{weight} function is a function $\delta_d:L\times W\times W\to\mathbb{R}$. We say that postulate non-negativity, identity, symmetry, weak inversity, strict inversity, equi-distance or faithfulness is satisfied by $\delta_d(\alpha,w,w')$ iff the respective postulate is satisfied by $\iota_d(w,w')$.
\end{definition}

To start off with, weight functions should be in the range $[0,1]$ so as to support the notion of the \textit{expectation} of probability. A weight function should thus be non-negative and have a maximum of 1. Moreover, a world's probability should carry the maximum weight when compared with itself (identity). Symmetry seems like a very natural property to expect of a weight function, especially given the close relationship that such functions have with distance. Weak inversity seems to be the weakest property that promotes some sort of inversity. It would be a useful distinction to be able to say that a weight function is weakly inverse, but not strictly so. Hence the definition of strict inversity. Knowing whether the property of equi-distance holds also seems useful and interesting.
Note that equi-distance implies symmetry due to symmetry of the pseudo-distance function, and by logic, equi-distance implies weak inversity.
Finally, whether a weight function is faithfulness seems like a natural question to ask.
And note that identity together with strict inversity implies faithfulness.

\begin{definition}
\label{def:EDI}
Let $b$ be an epistemic state and $\alpha$ a new piece of information. Then the new epistemic state changed with $\alpha$ via EDI is defined as
\[
b\,\mathsf{EDI}\,\alpha := \{(w,p)\mid p=0 \mbox{ if }w\not\Vdash\alpha,\mbox{ else }p=\frac{1}{\gamma}\sum_{w'\in W}b(w')\delta_d(\alpha,w,w')\},
\]
where $\gamma:=\sum_{w\in\mathit{Mod}(\alpha)}\sum_{w'\in W}b(w')\delta_d(\alpha,w,w')$ is a normalizing factor.

We may write $\mathsf{EDI}^*$ to indicate that $\delta_d$ is instantiated as $\delta_d^*$.
For any probabilistic belief change operator $\Delta$, we say that $\Delta$ is EDI-compatible if and only if there exists a function $\delta_d^*$ such that $b\,\Delta\,\alpha = b\,\mathsf{EDI}^*\,\alpha$ for all $b$ and $\alpha$.
\end{definition}

In the rest of the paper, we shall omit the subscript from $\delta_d$ as long is it is clear from the context or unnecessary to specify for which pseudo-distance function $\delta$ is defined.

\begin{definition}
\label{def:inv-dist-weight-func}
A weight function is an \emph{inverse-distance} weight function iff it satisfies postulates 1-4 above.
\end{definition}

Consider the following three properties a weight function might have.
\begin{itemize}
\item $\forall\alpha\in L, \forall w,w'\in W$, if $w\Vdash\alpha$ and $w'\Vdash\alpha$, then $\delta(\alpha,w,w')\neq0$ (Evidence Relaxation)
\item $\forall\alpha\in L, \forall w,w'\in W$, if $w\Vdash\alpha$ and $w'\not\Vdash\alpha$, then $\delta(\alpha,w,w')\neq0$ (Non-evidence Relaxation)
\item $\forall\alpha\in L, \forall w,w'\in W$, if $w\Vdash\alpha$ and $w'\Vdash\alpha$ and $w\neq w'$, then $\delta(\alpha,w,w')=0$ (Retention)
\end{itemize}
(In the following, an $\alpha$-world is a world satisfying the new evidence $\alpha$.)
We shall call a weight function \textit{e-relaxed} iff it satisfies evidence relaxation, and \textit{n-e-relaxed} iff it satisfies non-evidence relaxation. We shall say that a weight function is \textit{relaxed} iff it is both e-relaxed and n-e-relaxed.
Later, we shall show how $\alpha$-worlds `share' their mass with other $\alpha$-worlds during the change process if EDI is applied with an e-relaxed weight function. In such operations, belief mass tends to `spread' among worlds consistent with the same evidence (after repeated change operations).
Bayesian conditioning and the versions of imaging mentioned so far are not like this. In the latter operations, $\alpha$-worlds never give up their mass; they are `retentive'.
The idea captured by the retention condition says that when collecting probability mass for $w$ (because it is an $\alpha$ world), then do not collect anything from other $\alpha$ worlds. This rule affects equi-distance, weak and strict inversity.
In Section~\ref{sec:Probabilistic-Revision-via-Classical-Revision}, the reader will see that n-e-relaxed weight functions are also useful.

\begin{definition}
The EDI operation is said to be \emph{relaxed} or \emph{retentive} iff it satisfies the relaxation, respectively, retention property.
\end{definition}
Observe that a weight functions which satisfy e-relaxation cannot satisfy retention, and vice versa.

\subsection{Reciprocal Weights}

\citet{rm15b} define the weight to be applied to $b(w')$ when determining the new probability of $w$ as $\frac{1/d(w,w')}{\sum_{w''\in W, w''\neq w} 1/d(w,w'')}$. It is, however, ill-defined because it is undefined when $d(w,w')=0$ or $d(w,w'')=0$ (i.e., when $w=w'$ or $w=w''$).\footnote{Actually, \citet{rm15b} proposed two definitions for weights, one to be applied when ``adding worlds'' after a belief change and another when ``removing worlds''. In this paper, we use the on for adding worlds.} We thus adapt their definition based on the \textit{reciprocal} of distance as follows.
\begin{definition}
$\delta^\mathit{rcp}(\alpha,w,w') := \frac{\eta}{d(w,w')+\eta}$, for $\eta>0$, where $\eta$ is a real number.
\end{definition}
Because this definition is independent of evidence $\alpha$, we shall usually omit mentioning the evidence and simply write $\delta^\mathit{rcp}(w,w')$.

In the rest of this section, we use vocabulary $\{q,r\}$ and let $\eta=1$, unless stated otherwise. 
Recall that $b_{1.0}:=\langle 1,0,0,0\rangle$ and $b_{.3/.7}:=\langle 0.3,0.7,0,0\rangle$.
Two example applications of $\mathsf{EDI}^\mathit{rcp}$ follow.
\begin{example}~
\begin{itemize}
\item $b_{0.3/0.7}\:\mathsf{EDI}^\mathit{rcp}\:\lnot q(11)=0$
\item $b_{0.3/0.7}\:\mathsf{EDI}^\mathit{rcp}\:\lnot q(10)=0$
\item $b_{0.3/0.7}\:\mathsf{EDI}^\mathit{rcp}\:\lnot q(01)=\frac{1}{\gamma}[0.3\delta^\mathit{rcp}(01,11) + 0.7\delta^\mathit{rcp}(01,10) + 0\delta^\mathit{rcp}(01,01) + 0\delta^\mathit{rcp}(01,00)] = \frac{1}{\gamma}[0.3\times (1/2) + 0.7\times (1/3)] = \frac{1}{\gamma}[0.38\bar{3}] = 0.46$
\item $b_{0.3/0.7}\:\mathsf{EDI}^\mathit{rcp}\:\lnot q(00)=\frac{1}{\gamma}[0.3\delta^\mathit{rcp}(00,11) + 0.7\delta^\mathit{rcp}(00,10) + 0\delta^\mathit{rcp}(00,01) + 0\delta^\mathit{rcp}(00,00)] = \frac{1}{\gamma}[0.3\times (1/3) + 0.7\times (1/2)] = \frac{1}{\gamma}[0.45] = 0.54$
\end{itemize}
\end{example}
\begin{example}~
\begin{itemize}
\item $b_{1.0}\:\mathsf{EDI}^\mathit{rcp}\:\lnot q(11)=0$
\item $b_{1.0}\:\mathsf{EDI}^\mathit{rcp}\:\lnot q(10)=0$
\item $b_{1.0}\:\mathsf{EDI}^\mathit{rcp}\:\lnot q(01)=\frac{1}{\gamma}[1\delta^\mathit{rcp}(01,11) + 0\delta^\mathit{rcp}(01,10) + 0\delta^\mathit{rcp}(01,01) + 0\delta^\mathit{rcp}(01,00)] = \frac{1}{\gamma}[1\times (1/2) ] = \frac{1}{\gamma}[0.5] = 0.6$
\item $b_{1.0}\:\mathsf{EDI}^\mathit{rcp}\:\lnot q(00)=\frac{1}{\gamma}[1\delta^\mathit{rcp}(00,11) + 0\delta^\mathit{rcp}(00,10) + 0\delta^\mathit{rcp}(00,01) + 0\delta^\mathit{rcp}(00,00)] = \frac{1}{\gamma}[1\times (1/3)] = \frac{1}{\gamma}[0.\bar{3}] = 0.4$
\end{itemize}
\end{example}

\begin{proposition}
$\delta^\mathit{rcp}$ is an inverse-distance weight function, and also satisfies strict inversity, equi-distance and faithfulness, and is relaxed.
\end{proposition}
\begin{proof}
For all $w,w',w'',w'''\in W$:

Non-negativity: $d(w,w')\geq 0$ and $\eta>0$. Thus, $\delta^\mathit{rcp}(w,w') = \eta/(d(w,w')+\eta)>0$.

Identity: By the identity constraint of $d$, $d(w,w)=0$ for all $w\in W$. Hence, $\delta^\mathit{rcp}(w,w)$\\ $= \eta/(d(w,w)+\eta)=1$.

Inversity: Assume $d(w,w')>d(w'',w''')$. Then $\delta^\mathit{rcp}(w,w') = \eta/(d(w,w')+\eta) < \eta/(d(w'',w''')+\eta) = \delta^\mathit{rcp}(w'',w''')$. This implies that $\delta^\mathit{rcp}$ is weakly and strictly inverse.

Equi-distance: Assume $d(w,w')=d(w'',w''')$. Then $\delta^\mathit{rcp}(w,w') = \eta/(d(w,w')+\eta)=$ $\eta/(d(w'',w''')$ $+$ $\eta = \delta^\mathit{rcp}(w'',w''')$.

Symmetry is implied by equi-distance.

Faithfulness: Assume $w\neq w'$. Then, by the faithfulness of $d$, $d(w,w')>0$. Thus, $\delta^\mathit{rcp}(w,w') = \eta/(d(w,w')+\eta) < 1$.

For all $w,w'\in W$, $\eta/(d(w,w')+\eta)>0$. Relaxation (evidence and non-evidence) is thus satisfied.
\end{proof}

\subsection{Difference Weights}

A new definition with a similar meaning follows.It is also
inversely proportional to distance. It is based on the \textit{difference} between some particular applicable distance and the maximum distance. 
\begin{definition}
$\delta^\mathit{dfr}(\alpha,w,w') := 
\frac{d^\mathit{max}+\eta-d(w,w')}{d^\mathit{max}+\eta}$, for $\eta>0$, where $\eta$ is a real number and $d^\mathit{max}:=\max\{d(w,w')\mid w,w'\in W\}$.
\end{definition}
Because this definition is independent of evidence $\alpha$, we shall usually omit mentioning the evidence and simply write $\delta^\mathit{dfr}(w,w')$.
Two examples applications of $\mathsf{EDI}^\mathit{dfr}$ follow.
\begin{example}~
\begin{itemize}
\item $b_{0.3/0.7}\:\mathsf{EDI}^\mathit{dfr}\:\lnot q(11)=0$
\item $b_{0.3/0.7}\:\mathsf{EDI}^\mathit{dfr}\:\lnot q(10)=0$
\item $b_{0.3/0.7}\:\mathsf{EDI}^\mathit{dfr}\:\lnot q(01)=\frac{1}{\gamma}[0.3\delta^\mathit{dfr}(01,11) + 0.7\delta^\mathit{dfr}(01,10) + 0\delta^\mathit{dfr}(01,01) + 0\delta^\mathit{dfr}(01,00)] = \frac{1}{\gamma}[0.3\times (3-d(01,11))/3 + 0.7\times (3-d(01,10))/3] = \frac{1}{\gamma}[0.3 \times (2/3) + 0.7\times (1/3)] = \frac{1}{\gamma}[0.4\bar{3}] = 0.4\bar{3}$
\item $b_{0.3/0.7}\:\mathsf{EDI}^\mathit{dfr}\:\lnot q(00)=\frac{1}{\gamma}[0.3\delta^\mathit{dfr}(00,11) + 0.7\delta^\mathit{dfr}(00,10) + 0\delta^\mathit{dfr}(00,01) + 0\delta^\mathit{dfr}(00,00)] = \frac{1}{\gamma}[0.3\times (1/3) + 0.7\times (2/3)] = \frac{1}{\gamma}[0.5\bar{6}] = 0.5\bar{6}$
\end{itemize}
\end{example}
\begin{example}~
\begin{itemize}
\item $b_{1.0}\:\mathsf{EDI}^\mathit{dfr}\:\lnot q(11)=0$
\item $b_{1.0}\:\mathsf{EDI}^\mathit{dfr}\:\lnot q(10)=0$
\item $b_{1.0}\:\mathsf{EDI}^\mathit{dfr}\:\lnot q(01)=\frac{1}{\gamma}[1\delta^\mathit{dfr}(01,11) + 0\delta^\mathit{dfr}(01,10) + 0\delta^\mathit{dfr}(01,01) + 0\delta^\mathit{dfr}(01,00)] = \frac{1}{\gamma}[1\times (3-d(01,11))/3] = \frac{1}{\gamma}[1 \times (2/3)] = 0.\bar{6}$
\item $b_{1.0}\:\mathsf{EDI}^\mathit{dfr}\:\lnot q(00)=\frac{1}{\gamma}[1\delta^\mathit{dfr}(00,11) + 0\delta^\mathit{dfr}(00,10) + 0\delta^\mathit{dfr}(00,01) + 0\delta^\mathit{dfr}(00,00)] = \frac{1}{\gamma}[1\times (1/3)] = 0.\bar{3}$
\end{itemize}
\end{example}

If $\eta$ were allowed to equal zero, there would be cases when EDI-belief-change is undefined for $\delta$ instantiated as $\delta^\mathit{dfr}$. For instance, suppose $b(10)=1$ and the evidence is $\lnot q\land r$. Using Hamming distance for $d$, $\delta^\mathit{dfr}(01,10)=0$. Hence, $b \:\mathit{EDI}\:\lnot q\land r (01) = \frac{1}{\gamma}[0+b(10)\delta^\mathit{dfr}(01,10)+0+0] = \frac{1}{\gamma}0$, which is undefined because $\gamma=0$. But as soon as $\eta>0$, $\delta^\mathit{dfr}(01,10)>0$ and $b \:\mathit{EDI}\:\lnot q\land r (01) = 1$.

\begin{proposition}
$\delta^\mathit{dfr}$ is an inverse-distance weight function, and also satisfies strict inversity, equi-distance and faithfulness, and is relaxed.
\end{proposition}
\begin{proof}
For all $w,w',w'',w'''\in W$:

Non-negativity: By the definition of $d^\mathit{max}$, $d^\mathit{max}+\eta-d(w,w')\geq 0$, for all $w,w'\in W$. And due to the faithfulness condition of the pseudo-distance function, $d^\mathit{max}>0$. Thus, $\frac{d^\mathit{max}+\eta-d(w,w')}{d^\mathit{max}+\eta}\geq0$.

Identity: By the identity constraint of $d$, $d(w,w)=0$ for all $w\in W$. Hence, $\frac{d^\mathit{max}+\eta-d(w,w)}{d^\mathit{max}+\eta}=1$.

Equi-distance: Assume $d(w,w')=d(w'',w''')$. Then $\frac{d^\mathit{max}+\eta-d(w,w')}{d^\mathit{max}+\eta} = \frac{d^\mathit{max}+\eta-d(w'',w''')}{d^\mathit{max}+\eta}$.

Symmetry is implied by equi-distance.

Inversity: Assume $d(w,w')>d(w'',w''')$. Then $\delta^\mathit{dfr}(w,w') = \frac{d^\mathit{max}+\eta-d(w,w')}{d^\mathit{max}+\eta} < \frac{d^\mathit{max}+\eta-d(w'',w''')}{d^\mathit{max}+\eta} = \delta^\mathit{dfr}(w'',w''')$. This implies that $\delta^\mathit{dfr}$ is weakly and strictly inverse.

Faithfulness: Assume $w\neq w'$. Then, by the faithfulness of $d$, $d(w,w')>0$. Thus, $\delta^\mathit{dfr}(w,w') = \frac{d^\mathit{max}+\eta-d(w,w')}{d^\mathit{max}+\eta} < 1$.

For all $w,w'\in W$, $d^\mathit{max}+\eta-d(w,w)>0$. Relaxation (evidence and non-evidence) is thus satisfied.
\end{proof}

\subsection{Bayesian Conditioning i.t.o. EDI}

Bayesian conditioning can be nicely simulated as an EDI operator.
Let $\delta^\mathit{BC}$ be defined as follows.
\begin{equation*}
\delta^\mathit{BC}(\alpha,w,w'):=\left\lbrace
\begin{array}{ll}
1 & \textit{if }w=w'\\
0 & \textit{otherwise}
\end{array}
\right.
\end{equation*}
Because the definition of $\delta^\mathit{BC}$ is independent of evidence $\alpha$, we shall usually omit mentioning the evidence and simply write $\delta^\mathit{BC}(w,w')$.

\begin{proposition}
$b\,\mathsf{BC}\,\alpha = b\,\mathsf{EDI}^\mathit{BC}\,\alpha$ iff $b(\alpha)>0$. That is, $\mathit{BC}$ is EDI-compatible iff $b(\alpha)>0$.
\end{proposition}
\begin{proof}
\begin{align*}
b\,\mathsf{BC}\,\alpha (w) &= \frac{b(w,\alpha)}{b(\alpha)},b(\alpha)>0\\
& =\left\lbrace
\begin{array}{ll}
0 & \textit{if }w\not\Vdash\alpha\\
\frac{b(w)}{b(\alpha)} & \textit{otherwise}
\end{array}
\right.\\
& =\left\lbrace
\begin{array}{ll}
0 & \textit{if }w\not\Vdash\alpha\\
\frac{1}{\gamma}b(w) & \textit{otherwise, where }\gamma=b(\alpha)
\end{array}
\right.\\
& =\left\lbrace
\begin{array}{ll}
0 & \textit{if }w\not\Vdash\alpha\\
\frac{1}{\gamma}\sum_{w'\in W}b(w')\delta^\mathit{BC}(w,w') & \textit{otherwise},
\end{array}
\right.\\
&\textit{where }\gamma=\sum_{w\in\mathit{Mod}(\alpha)}\sum_{w'\in W}b(w')\delta^\mathit{BC}(w,w')
\end{align*}
implying that $b\,\mathsf{BC}\,\alpha =  \{(w,p)\mid p=0 \mbox{ if }w\not\Vdash\alpha, \mbox{ else }p=\frac{1}{\gamma}\sum_{w'\in W}b(w')\delta^\mathit{BC}(w,w')\} = b\,\mathsf{EDI}^\mathit{BC}\,\alpha$, for $b(\alpha)>0$.
\end{proof}

\begin{proposition}
$\delta^\mathit{BC}$ is an inverse-distance weight function satisfying equi-distance, symmetry, retention and faithfulness, but not strict inversity.
\end{proposition}
\begin{proof}
For all $w,w',w'',w'''\in W$:

Clearly, $\delta^\mathit{BC}(w,w')$ is non-negative.

Identity follows directly from the definition.

Weak Inversity: Assume $d(w,w')>d(w'',w''')$. There are two cases to consider: Either $d(w,w')>d(w'',w''')=0$ or $d(w,w')>d(w'',w''')\neq0$. Assume $d(w,w')>d(w'',w''')=0$. Then $w''=w'''$ and $w\neq w'$, in which case, $0=\delta^\mathit{BC}(w,w')<\delta^\mathit{BC}(w'',w''')=1$. Now assume $d(w,w')>d(w'',w''')\neq0$. Then $w''\neq w'''$ and $w\neq w'$, in which case, $\delta^\mathit{BC}(w,w')=\delta^\mathit{BC}(w'',w''')=0$. By combining the two cases, one sees that the desired result follows.
$\delta^\mathit{BC}$ is thus an inverse-distance weight function.

Equi-distance: Assume $d(w,w')=d(w'',w''')$. There are two cases to consider: Either $d(w,w')=d(w'',w''')=0$ or $d(w,w')=d(w'',w''')\neq0$. Assume $d(w,w')=d(w'',w''')=0$. Then $w=w'$ and $w''=w'''$, in which case, $\delta^\mathit{BC}(w,w') =$ $\delta^\mathit{BC}(w'',w''')$ $=$ $1$.
Now assume $d(w,w')=d(w'',w''')\neq0$. Then $w\neq w'$ and $w''\neq w'''$, in which case, $\delta^\mathit{BC}(w,w') = \delta^\mathit{BC}(w'',w''') = 0$.

Symmetry is implied by equi-distance.

Faithfulness follows directly from the definition.

Assume again that $d(w,w')>d(w'',w''')\neq0$. Then $w''\neq w'''$ and $w\neq w'$, in which case, $\delta^\mathit{BC}(w,w')$ $=$\\ $\delta^\mathit{BC}(w'',w''')=0$. Thus strict inversity does not hold.
\end{proof}

\begin{proposition}
$\delta^\mathit{BC}$ is retentive.
\end{proposition}
\begin{proof}
For all $w,w',w'',w'''\in W$:

Assume $w\Vdash\alpha$ and $w'\Vdash\alpha$ and $w\neq w'$. Then, immediately, by the definition of $\delta^\mathit{BC}$, $\delta^\mathit{BC}(w,w')=0$.
\end{proof}

\subsection{Lewis Imaging i.t.o. EDI}

Lewis imaging can be simulated as an EDI operator, but not as cleanly as $\mathsf{BC}$.
Let $\delta^\mathit{LI}$ be defined as follows.
\begin{equation*}
\delta^\mathit{LI}(\alpha,w,w'):=\left\lbrace
\begin{array}{ll}
x & \textit{if }w\not\Vdash\alpha\\
1 & \textit{if }w\in\mathit{Min}(\alpha,w',d^\mathit{LI})\\
0 & \textit{otherwise},
\end{array}
\right.
\end{equation*}
where $x$ is fixed between 0 and 1, inclusive, and pseudo-distance function $d^\mathit{LI}$ induces a mapping from worlds to total orders (instead of the weaker total pre-orders).


\begin{proposition}
\label{prp:LI=EDI}
$b\,\mathsf{LI}\,\alpha = b\,\mathsf{EDI}^\mathit{LI}\,\alpha$.
That is, $\mathit{LI}$ is EDI-compatible.
\end{proposition}
\begin{proof}
Let $w\not\Vdash\alpha$. Then, independent of the definition of $\delta^\mathit{LI}$, $b\,\mathsf{LI}\,\alpha(w) = b\,\mathsf{EDI}^\mathit{LI}\,\alpha(w) = 0$.

Now let $w\Vdash\alpha$. Then,
\begin{eqnarray*}
b\,\mathsf{LI}\,\alpha(w) &=& \sum_{\substack{w'\in W\\w'^\alpha=w}}b(w')\\
&=& \sum_{w'\in W}b(w')\delta^\mathit{LI-}(w,w'),
\end{eqnarray*}
where
\begin{eqnarray*}
\delta^\mathit{LI-}(w,w')&:=&\left\lbrace
\begin{array}{ll}
1 & \textit{if } w'^\alpha=w\\
0 & \textit{otherwise}
\end{array}
\right.\\
&=& \left\lbrace
\begin{array}{ll}
1 & \textit{if } w\in\mathit{Min}(\alpha,w',d^\mathit{LI})\\
0 & \textit{otherwise}
\end{array}
\right.
\end{eqnarray*}
Noting that the first line of the definition of $\delta^\mathit{LI}$ is applicable only if $w\not\Vdash\alpha$, then the result follows when combining the cases when $w\not\Vdash\alpha$ and $w\Vdash\alpha$.
\end{proof}

Whereas Proposition~\ref{prp:LI=EDI} says that there exists a weight function for simulating Lewis imaging with EDI, Proposition~\ref{def:LI-no-inv-dist-weight-func} says that the function cannot be \textit{inverse-distance}.
\begin{proposition}
\label{def:LI-no-inv-dist-weight-func}
There exists no inverse-distance weight function $\delta^*$ for which $b\,\mathsf{LI}\,\alpha = b\,\mathsf{EDI}^*\,\alpha$.
\end{proposition}
\begin{proof}
We provide an example for which weak inversity must fail. $\delta^\mathit{LI}$ is thus not an inverse-distance weight function for this example, nor in general.

Recall that $d^\mathit{LI}$ is such that for every $\alpha$, for each of the worlds, a unique closest $\alpha$-world can be identified.

Observe that $\mathsf{EDI}^\mathit{LI}$ is retentive: Recall the retention property: if $w\Vdash\alpha$ and $w'\Vdash\alpha$ and $w\neq w'$, then $\delta^\mathit{LI}(\alpha,w,w')=0$. Directly, the first line of the definition of $\delta^\mathit{LI}$ is not applicable, and it is impossible for $w\in\mathit{Min}(\alpha,w',d^\mathit{LI})$ \underline{and} $w'\Vdash\alpha$ and $w\neq w'$, making the second line inapplicable. Only the third line is applicable, and retention thus holds.

Now let $d^\mathit{LI}(w,w')=d^\mathit{LI}(w'',w''')$. 
And let $w,w'',w'''\Vdash\alpha$. Then $\mathit{Min}(w,\alpha,d^\mathit{LI})=\{w\}$,\\ $\mathit{Min}(w'',\alpha,d^\mathit{LI})=\{w''\}$, $\mathit{Min}(w''',\alpha,d^\mathit{LI})=\{w'''\}$ and let $\mathit{Min}(w',\alpha,d^\mathit{LI})=\{w\}$.
Note that $\delta^\mathsf{LI}(\alpha,w,w')$ must equal 1 (second line of definition) and by retention, $\delta^\mathit{LI}(\alpha,w'',w''') = 0$. Therefore, $d^\mathit{LI}(w,w')$\\ $\geq d^\mathit{LI}(w'',w''')$, but $\delta^\mathit{LI}(\alpha,w,w')\not\leq\delta^\mathit{LI}(\alpha,w'',w''')$.
\end{proof}

\begin{proposition}
$\delta^\mathit{LI}$ is retentive.
\end{proposition}
\begin{proof}
Assume $w,w'\Vdash\alpha$ and $w\neq w'$.
Only the second and third lines of the definition of $\delta^\mathit{LI}$ are applicable.
We know that whenever $w''\Vdash\alpha$, then $\mathit{Min}(w'',\alpha,d^\mathit{LI})=\{w''\}$. Hence, due to $w\neq w'$, the second line cannot be applicable.
Therefore, $\delta^\mathit{LI}(\alpha,w,w')=0$.
\end{proof}

\subsection{Generalized Imaging i.t.o. EDI}

Generalized imaging can be simulated as an EDI operator, and where $\mathsf{LI}$ is always retentive,  $\mathsf{GI}$ is only retentive under a reasonable condition.
Let $\delta^\mathit{GI}$ be defined as follows.
\begin{equation*}
\delta^\mathit{GI}(\alpha,w,w'):=\left\lbrace
\begin{array}{ll}
1 & \textit{if }w=w' \textit{ and } w\not\Vdash\alpha\\
1/|\mathit{Min}(\alpha,w',d)| & \textit{if }w\in\mathit{Min}(\alpha,w',d)\\
0 & \textit{otherwise}
\end{array}
\right.
\end{equation*}

\begin{proposition}
\label{prp:GI=EDI}
$b\,\mathsf{GI}\,\alpha = b\,\mathsf{EDI}^\mathit{GI}\,\alpha$. That is, 
$\mathit{GI}$ is EDI-compatible.
\end{proposition}
\begin{proof}
Let $w\not\Vdash\alpha$. Then, independent of the definition of $\delta^\mathit{GI}$, $b\,\mathsf{GI}\,\alpha(w) = b\,\mathsf{EDI}^\mathit{GI}\,\alpha(w) = 0$.

Now let $w\Vdash\alpha$. Then,
\begin{eqnarray*}
b\,\mathsf{GI}\,\alpha(w) &=& \sum_{\substack{w'\in W\\w\in\mathit{Min}(\alpha,w',d)}}b(w')/|\mathit{Min}(\alpha,w',d)|\\
&=& \sum_{w'\in W}b(w')\delta^\mathit{GI-}(w,w'),
\end{eqnarray*}
where
\begin{equation*}
\delta^\mathit{GI-}(w,w'):=\left\lbrace
\begin{array}{ll}
1/|\mathit{Min}(\alpha,w',d)| & \textit{if }w\in\mathit{Min}(\alpha,w',d)\\
0 & \textit{otherwise}
\end{array}
\right.
\end{equation*}
Noting that the first line of the definition of $\delta^\mathit{GI}$ is applicable only if $w\not\Vdash\alpha$, then the result follows when combining the cases when $w\not\Vdash\alpha$ and $w\Vdash\alpha$.
\end{proof}
For example, the reader can confirm that $b_{0.3/0.7}\,\mathsf{GI}\,\lnot q = b_{0.3/0.7}\,\mathsf{EDI}^\mathit{GI}\,\lnot q$ and $b_{1.0}\,\mathsf{GI}\,\lnot q =$ $b_{1.0}\,\mathsf{EDI}^\mathit{GI}\,\lnot q$.

Whereas Proposition~\ref{prp:GI=EDI} says that there exists a weight function for simulating generalized imaging with EDI, Proposition~\ref{prp:GI-no-inv-dist-weight-func} says that the function cannot be \textit{inverse-distance}.
\begin{proposition}
\label{prp:GI-no-inv-dist-weight-func}
There exists no inverse-distance weight function $\delta^*$ for which $b\,\mathsf{GI}\,\alpha = b\,\mathsf{EDI}^*\,\alpha$.
\end{proposition}
\begin{proof}
Lewis imaging is a specialization of generalized imaging. The proposition is thus entailed by Proposition~\ref{def:LI-no-inv-dist-weight-func}.
\end{proof}

\begin{lemma}
\label{lm:singleton-iff-faithful}
Let $w\Vdash\alpha$.
Then, $\forall w\in W, \mathit{Min}(w,\alpha,d)=\{w\}$ iff $d$ is faithful.
\end{lemma}
\begin{proof}
Recall that $\mathit{Min}(\alpha,w,d)=\{w'\Vdash\alpha\mid \forall w''\Vdash\alpha, d(w',w)\leq d(w'',w)\}$.

Assume $d$ is faithful.
Then for all $w,w'\in W$, if $w\neq w'$, then $d(w, w') > 0$.
By identity of $d$, $d(w,w)=0$.
Hence, $\mathit{Min}(w,\alpha,d)=\{w\}$.

Assume $d$ is not faithful.
Then there exists a pair of worlds $w,w'\in W$, such that if $w\neq w'$, then $d(w, w') = 0$.
Hence, $\{w,w'\}\subseteq\mathit{Min}(w,\alpha,d)$.
Therefore, it is not the case that $\mathit{Min}(w,\alpha,d)=\{w\}$.
\end{proof}

\begin{proposition}
$\delta_d^\mathit{GI}$ is retentive iff $d$ is faithful.
\end{proposition}
\begin{proof}
Assume $d$ is faithful.
Assume $w,w'\Vdash\alpha$ and $w\neq w'$.
Only the second and third lines of the definition of $\delta^\mathit{GI}$ are potentially applicable.
By Lemma~\ref{lm:singleton-iff-faithful}, we know that whenever $w\Vdash\alpha$, then $\mathit{Min}(w',\alpha,d)=\{w\}$ such that $w=w'$. Hence, due to our assumption that $w\neq w'$, the second line cannot be applicable.
Therefore, $\delta^\mathit{GI}(\alpha,w,w')=0$.

Assume $d$ is not faithful.
Assume $w,w'\Vdash\alpha$ and $w\neq w'$.
Only the second and third lines of the definition of $\delta^\mathit{GI}$ are potentially applicable.
By Lemma~\ref{lm:singleton-iff-faithful}, we know that whenever $w\Vdash\alpha$, then $\{w,w'\}\subseteq\mathit{Min}(w',\alpha,d)$ such that $w\neq w'$. Hence, due to our assumption that $w\neq w'$, the second line might be applicable.
Therefore, it could be that $\delta^\mathit{GI}(\alpha,w,w')=1/|\mathit{Min}(\alpha,w',d)>0$.
\end{proof}
Faithfulness is a reasonable property to expect of a distance function. $\delta_d^\mathit{GI}$ is thus retentive under reasonable conditions.

\subsection{Discussion}

To summarize, $\delta^\mathit{rcp}$ and $\delta^\mathit{dfr}$ are relaxed, and $\delta^\mathit{BC}$ and $\delta^\mathit{LI}$ are retentive, and $\delta^\mathit{GI}$ is retentive iff the associated pseudo-distance function is faithful.

Let $\delta^\mathit{rlx}$ be $\delta^\mathit{rcp}$ or $\delta^\mathit{dfr}$.
Let $\alpha$ be a particular piece of evidence and $b_t$ the belief state resulting from the $t$-th repeated application of~~$\mathsf{EDI}^\mathit{rlx}$ on the resulting belief states for $\alpha$.
That is, $b_1 = b\,\mathsf{EDI}^\mathit{rlx}\,\alpha$, $b_2 = b_1\,\mathsf{EDI}^\mathit{rlx}\,\alpha$, and so on.
By running multiple experiments on a computer, we noticed that, as $t$ increases, $b_t$ settles on a particular belief state, that is, $b_{t-1}\approx b_t$. (Belief states had eight worlds.)
We ran 100 trials for each of the four combinations of choices for $\delta^\mathit{rlx}$ and $\eta=1$ or $\eta=0.0001$. For each trial, an initial belief state and evidence was chosen randomly. $\alpha\equiv\top$ and $\alpha\equiv\bot$ were ignored. Per trial, operation $\mathsf{EDI}$ was applied to the resulting belief state ten times with the $\alpha$ selected for that trial. The difference between the probabilities of a world at successive applications of $\mathsf{EDI}$ were recorded, and the difference of all worlds averaged.
See Figure~\ref{fig:diff-results}. The figure shows the average for the 100 trials of the average difference at each application of the change operation. Note how the differences asymptotically decrease. Note that these results do not imply that belief states become uniform distributions over $\alpha$-worlds. In fact, distributions are typically not uniform; distributions seem to be dictated by the mutual reinforcement of $\alpha$-world probabilities and their corresponding distances from each other. However, whenever exactly two worlds model $\alpha$, the probability is eventually divided exactly (uniformly) between the two worlds.

\begin{figure}[t]
\centering
\includegraphics[scale=0.5]{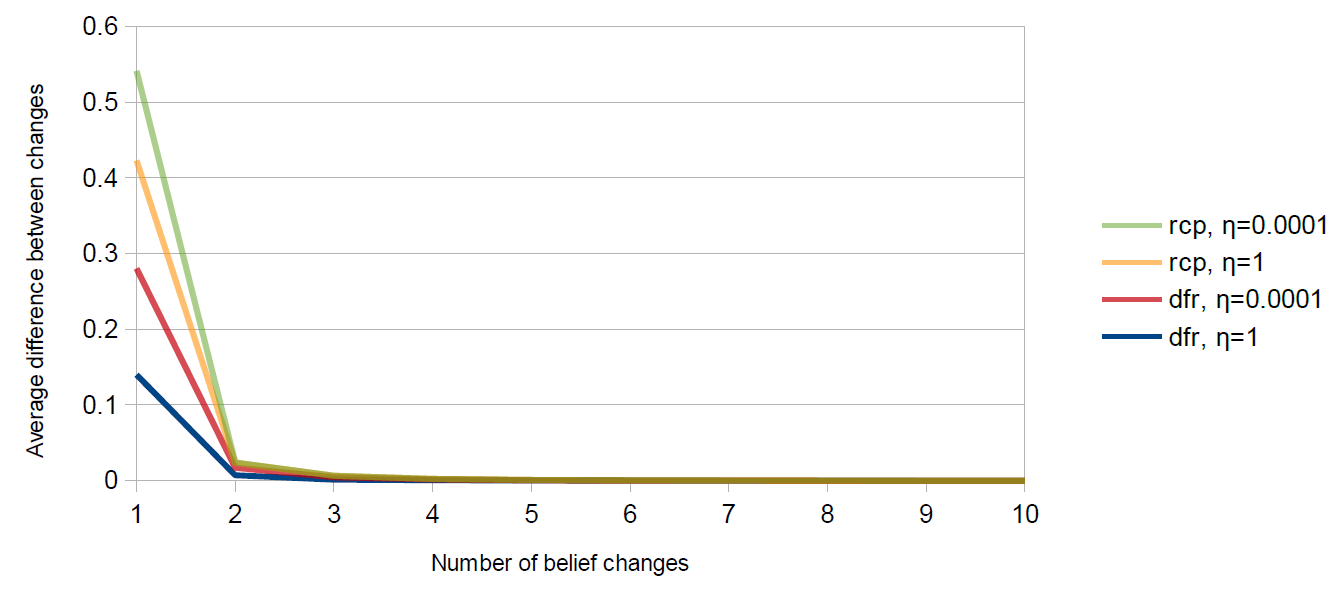}
\caption{The averages of the differences in probability distributions per belief change operation. Results are the averages of 100 randomly generated cases.\label{fig:diff-results}}
\end{figure}

\begin{proposition}
Let $\alpha$ be a particular piece of evidence and $b_t$ the belief state resulting from the $t$-th repeated application of~~$\mathsf{EDI}^\mathit{rtv}$ on the resulting belief states for $\alpha$, where $\delta^\mathit{rtv}$ is a retentive inverse-distance weight function.
That is, $b_1 = b\,\mathsf{EDI}^\mathit{rtv}\,\alpha$, $b_2 = b_1\,\mathsf{EDI}^\mathit{rtv}\,\alpha$, and so on. Then $b_t = b_1$ for all $t>1$.
\end{proposition}
\begin{proof}
By definition of $\mathsf{EDI}^\mathit{rtv}$, if $w\not\Vdash\alpha$, then $b_1(w) = 0$, that is, if $b_1(w) > 0$, then $w\Vdash\alpha$. Therefore, we are only interested in the potential change in probability of $\alpha$-worlds. Hence, we consider only $\delta^\mathit{rtv}(w,w')$ such that $w\Vdash\alpha$. And because $w'$ contributes no mass to $w$ if $w'\not\Vdash\alpha$, we consider only $\delta^\mathit{rtv}(w,w')$ such that $w,w'\Vdash\alpha$. Then by retention and identity of an inverse-distance weight function, $\delta^\mathit{rtv}(w,w') = 1$ iff $w=w'$, else it equals 0. Hence, for all $w\Vdash\alpha$, $b_t(w) = \sum_{w'\in W}b_{t-1}(w')\delta^\mathit{rtv}(w,w') = b_{t-1}(w)$ for all $t>1$.
\end{proof}

The difference between retention and relaxation seems important, from the perspective of (formal) epistemology.
Suppose $w_1\Vdash\alpha$ and $w_2\Vdash\beta$ have probabilities 0.4, respectively, 0.6. Now if $\alpha\land\beta$ is successively observed, then according to, for instance, Bayesian conditioning and Lewis imaging, the probabilities of $w_1$ and $w_2$ will not change. We argue that there are situations when their probabilities should be allowed to change. One could argue that if $\alpha$ and $\beta$ are always observed together, then the probabilities of $w_1$ and $w_2$ should become equal.
To put it differently, the fact that the same information is continually received should make a difference to the spread of the likelihood of the constituent parts (atoms) of the information. 
Nonetheless, information received in a static environment has a different flavor to information received in a dynamic environment (giving rise to belief revision, resp., belief update):
Repeated observation of the same piece of information $\alpha'$ should not change beliefs therein, beyond the first revision. The first in the series of observations revises beliefs, but subsequent observations of $\alpha'$ have no effect because it is `old news'/already learnt. In this sense, revision corresponds to retentive belief change. On the other hand, when the same signal is observed in the environment, every new observation is a \textit{further} confirmation of the truth of $\alpha'$, thus modifying one's belief in the atoms of $\alpha'$ with every new observation instance (of the same $\alpha'$). In this sense, update corresponds to evidence-relaxed belief change.

\section{EDI for Revision vs. EDI for Update}

In this section, we couch all belief change operations as EDI. For each of revision and update,  EDI will be specialized into two proposed definitions. One of the specializations in each case will have the following pattern: Apply the classical operation to determine the newly believed worlds, and then use EDI to determine a probability distribution over them. The other specialization in each case will be a more direct use of EDI. We shall also check, for each of the four operations, which properties the corresponding weight function satisfies. We end this section with a discussion of what it takes to translate classical belief change as EDI in terms of uniform probability distributions.

\citet{km92} use an example involving a room with a table, a book and a magazine in it to illustrate the difference between revision and update. Suppose we only know that either the book is on the table or the magazine is on the table, but not both. Let $\mathit{book}$ mean the book is on the floor and $\mathit{mag}$ mean the magazine is on the floor. Then our belief state can reasonably be represented as $b=\langle0,0.5,0.5,0\rangle$, that is $b(\mathit{book}\land\lnot\mathit{mag})=b(\lnot\mathit{book}\land\mathit{mag})=0.5$ and $b(\mathit{book}\land\mathit{mag})=b(\lnot\mathit{book}\land\lnot\mathit{mag})=0$.
\citet{km92} argued that after revision with $\mathit{book}$, one should believe $\mathit{book}\land\lnot\mathit{mag}$.
We thus want
\[
b\,\mathsf{EDI}^\circ\,\mathit{book}=\langle0,1,0,0\rangle,
\]
where $\circ$ indicates revision. And they argued that after update with $\mathit{book}$, one should believe $\mathit{book}$. We thus want
\[
b\,\mathsf{EDI}^\diamond\,\mathit{book}=\langle0.5,0.5,0,0\rangle,
\]
where $\diamond$ indicates update.
This is the case when $\circ$ is $\mathsf{BC}$, and when $\diamond$ is $\mathit{rcp}$ or $\mathit{dfr}$ after an infinite number of $\mathsf{EDI}^\diamond$ with $\mathit{book}$.
However, in this example, $b(\mathit{book})>0$.
The question is, What should probabilistic revision (and update) correspond to when $b(\mathit{book})=0$? For instance, what if $\alpha=\mathit{book}\leftrightarrow\mathit{mag}$, given the agent currently believes $\mathit{book}\leftrightarrow\lnot\mathit{mag}$? Furthermore, we would like the same operation to be applicable whether or not the evidence contradicts current beliefs.

\subsection{Probabilistic Revision via Classical Revision}
\label{sec:Probabilistic-Revision-via-Classical-Revision}

This approach is to determine the new knowledge base exactly as one would in classical belief revision, and then to assign probabilities to the newly believed worlds.

Let
\[
W^b := \{w\in W\mid (w,p)\in b, p\neq 0\}.
\]
Let $\psi^b$ be a sentence modeled by all worlds in $W^b$ and no other worlds.
Let $\delta^{\widehat{\mathit{ClsRev}}}$ be defined as follows.
\begin{equation*}
\delta^{\widehat{\mathit{ClsRev}}}(\alpha,w,w'):=\left\lbrace
\begin{array}{ll}
1 & \textit{if }w=w'\\
\delta(\alpha,w,w') & \textit{if }w\in\mathit{Mod}(\psi^b\circ\alpha)\\
0 & \textit{otherwise},
\end{array}
\right.
\end{equation*}
where $\circ$ is some (acceptable) revision operator and $\delta$ is a retentive inverse-distance weight function.

Unfortunately, $\delta^{\widehat{\mathit{ClsRev}}}$ is not well defined, because $b\,\mathsf{EDI}^{\widehat{\mathit{ClsRev}}}\,\alpha$ might not be a belief state. For example, suppose there are four worlds $w_1$, $w_2$, $w_3$ and $w_4$, where $b(w_1)=1$ and only $w_4\Vdash\alpha$ ($w_1,w_2,w_3\not\Vdash\alpha$). Let $ \delta$ be an inverse-distance weight function such that $\delta(\alpha,w_1,w_4)$ $=$ $\delta(\alpha,w_4,w_1)$ $=$ $0$.
Then, intuitively, we want $b\,\mathsf{EDI}^{\widehat{\mathit{ClsRev}}}\,\alpha(w_4)=1$.
By definition, $b\,\mathsf{EDI}^{\widehat{\mathit{ClsRev}}}\,\alpha(w_i)=0$ for $i=1,2,3$.
We know that $\mathit{Mod}(\psi^b\circ\alpha)=\mathit{Mod}(\alpha)=\{w_4\}$. To compute $b\,\mathsf{EDI}^{\widehat{\mathit{ClsRev}}}\,\alpha(w_4)$, the second line of the definition of $\delta^{\widehat{\mathit{ClsRev}}}$ is applicable, hence,
\begin{eqnarray*}
b\,\mathsf{EDI}^{\widehat{\mathit{ClsRev}}}\,\alpha(w_4)&=&\frac{1}{\gamma}\sum_{w'\in W}b(w')\delta(\alpha,w_4,w')\\
&=& \frac{1}{\gamma}\big(b(w_1)\delta(\alpha,w_4,w_1)+b(w_2)\delta(\alpha,w_4,w_2)+b(w_3)\delta(\alpha,w_4,w_3)+b(w_4)\delta(\alpha,w_4,w_4)\big)\\
&=& \frac{1}{\gamma}\big(1\times0+0\times?+0\times?+0\times?\big).
\end{eqnarray*}

This problem would not occur if $\delta(\alpha,w_1,w_4)=0$ were not allowed.
We thus define $\delta^\mathit{ClsRev}$ as
\begin{equation*}
\delta^\mathit{ClsRev}(\alpha,w,w'):=\left\lbrace
\begin{array}{ll}
1 & \textit{if }w=w'\\
\delta(\alpha,w,w') & \textit{if }w\in\mathit{Mod}(\psi^b\circ\alpha)\\
0 & \textit{otherwise},
\end{array}
\right.
\end{equation*}
where $\circ$ is some (acceptable) revision operator and $\delta$ is a non-evidence (n-e-) relaxed and retentive inverse-distance weight function.

\begin{proposition}
$\delta^\mathit{ClsRev}$ is retentive.
\end{proposition}
\begin{proof}
Assume $w,w'\Vdash\alpha$ and $w\neq w'$.
Then only the second and third lines of the definition of $\delta^\mathit{ClsRev}$ are applicable.
If $w\not\in\mathit{Mod}(\psi^b\circ\alpha)$, then the third line would be applicable, making $\delta^\mathit{ClsRev}(\alpha,w,w')=0$.
Thus, the proposition's veracity depends only on the value of $\delta(\alpha,w,w')$. And due to the condition (``if $w,w'\Vdash\alpha \textit{ and } w\neq w'$'') in the first line of the definition of $\delta$, $\delta^\mathit{ClsRev}(\alpha,w,w')$ $=$ $\delta'(\alpha,w,w')$ $=$ $0$.
\end{proof}

\begin{proposition}
$\delta^\mathit{ClsRev}$ satisfies non-negativity and identity.
\end{proposition}
\begin{proof}
Non-negativity and identity follow directly from the definition of $\delta^\mathit{ClsRev}$.
\end{proof}

\begin{proposition}
$\delta^\mathit{ClsRev}$ does not satisfy symmetry, weak inversity, strong inversity nor equi-distance.
\end{proposition}
\begin{proof}
Suppose the vocabulary has three atoms, the distance function is Hamming and
\begin{equation*}
\delta(\alpha,w,w'):=\left\lbrace
\begin{array}{ll}
0 & \textit{if }w,w'\Vdash\alpha \textit{ and } w\neq w'\\
\delta^\mathit{rcp}(\alpha,w,w') & \textit{otherwise},
\end{array}
\right.
\end{equation*}
with $\eta=1$. Note that $\delta$ is an n-e-relaxed and retentive inverse-distance weight function.
Assume $\mathit{Mod}(\psi^b) = \{101,001,000\}$ and $111,110,011,010\Vdash\alpha$ and all other worlds do not model $\alpha$.
Finally, assume that $\circ$ is defined such that $\mathit{Mod}(\psi^b\circ\alpha)=\{111,011,010\}$.
Note that $\mathit{Mod}(\psi^b\circ\alpha)\subseteq\mathit{Mod}(\alpha)$.

Then $d(010,000)=d(110,010)=1$,
$\delta^\mathit{ClsRev}(\alpha,010,000)$ $=$ $\delta(\alpha,010,000)$ $=$ $\delta^\mathit{rcp}(\alpha,010,000)$ $=$ $1/2$ and\\ $\delta^\mathit{ClsRev}(\alpha,110,010) = 0$.
Therefore, $d(010,000)\geq d(110,010)$, but $\delta^\mathit{ClsRev}(\alpha,010,000)$ $\not\leq$\\ $ \delta^\mathit{ClsRev}(\alpha,110,010)$. Hence, weak inversity fails for $\delta^\mathit{ClsRev}$. This case also proves that equi-distance fails.

Due to strict inversity implying weak inversity, by contraposition, strict inversity fails.

To check symmetry, observe that $\delta^\mathit{ClsRev}(\alpha,000,010) = 0$ because $000\neq 010$ and $000\not\in \mathit{Mod}(\psi^b\circ\alpha)$. But as we see above, $\delta^\mathit{ClsRev}(\alpha,010,000) = 1/2$. Thus, symmetry fails.
\end{proof}

\begin{corollary}
Because $\delta^\mathit{ClsRev}$ is not symmetric or weakly inverse, it is not an inverse-distance weight function.
\end{corollary}

\begin{proposition}
$\delta^\mathit{ClsRev}$ is faithful iff $\delta$ is.
\end{proposition}
\begin{proof}
Assume $w\neq w'$.

Assume $\delta$ is not faithful.
It could thus happen that $\delta(\alpha,w,w')=1$.
Next, assume that $w\in\mathit{Mod}(\psi^b\circ\alpha)$ and $w'\not\Vdash\alpha$, implying that $\delta^\mathit{ClsRev}(\alpha,w,w') = \delta(\alpha,w,w')$. Hence, $\delta^\mathit{ClsRev}$ is not faithful.

Now assume $\delta$ is faithful.
Only the second and third lines are applicable.
In both cases,\\	 $\delta^\mathit{ClsRev}(\alpha,w,w')$ $<$ $1$ (particularly, $\delta(\alpha,w,w')<1$, because$w\neq w'$).
\end{proof}

\subsection{Probabilistic Revision with EDI Directly}
\label{sec:Probabilistic-Revision-with-EDI-Directly}

The approach here is that revision should satisfy the retentive property.
Let
\begin{equation*}
\delta^{=0}(\alpha,w,w'):=\left\lbrace
\begin{array}{ll}
1 & \textit{if } w=w'\\
\delta(\alpha,w,w') & \textit{if } w\not\Vdash\alpha \textit{ or } w'\not\Vdash\alpha, \textit{ else}\\
0 & \textit{otherwise}
\end{array}
\right.
\end{equation*}
where $\delta$ is an inverse-distance weight function.
Note that $\delta^{=0}(\alpha,w,w')=0$ whenever $w,w'\Vdash\alpha$ and $w\neq w'$ (retention).

\begin{proposition}
$\delta^{=0}(\alpha,w,w')=0$ does not satisfy weak inversity.
\end{proposition}
\begin{proof}
Assume $d(w,w')=d(w'',w''')\neq0$. Then, by the identity condition of $d$, $w\neq w'$ and $w''\neq w'''$.
Assume $w'\not\Vdash\alpha$, and $w'',w'''\Vdash\alpha$.
Then $\delta^{=0}(\alpha,w,w')=\delta(\alpha,w,w')$ and $\delta^{=0}(\alpha,w'',w''')=0$.
Now it could happen that $\delta(\alpha,w,w')>0$, in which case, $d(w,w')\geq d(w'',w''')$, but $\delta^{=0}(\alpha,w,w')\not\leq \delta^{=0}(\alpha,w'',w''')$.
\end{proof}

$\delta^{=0}$ does not work as desired when $b(\alpha)>0$ ($\delta$ is instantiated as $\delta^\mathit{dfr}$ and $\eta=1$):
\begin{eqnarray*}
b^{\mathsf{EDI}^{=0}}_\mathit{book}(11) &=& \frac{1}{\gamma}\Big[b(10)\delta^{=0}(\alpha,11,10) +\\
&& b(01)\delta^{=0}(\alpha,11,01)\Big]\\
&=& \frac{1}{\gamma}\Big[0.5\times0 + 0.5\times\frac{2+1-1}{2+1}\Big]\\
&=& \frac{1}{\gamma}\Big[\frac{2}{3}\Big] = 0.5
\end{eqnarray*}
\begin{eqnarray*}
b^{\mathsf{EDI}^{=0}}_\mathit{book}(10) &=& \frac{1}{\gamma}\Big[b(10)\delta^{=0}(\alpha,10,10) +\\
&& b(01)\delta^{=0}(\alpha,10,01)\Big]\\
&=& \frac{1}{\gamma}\Big[0.5\times1 + 0.5\times\frac{2+1-2}{2+1}\Big]\\
&=& \frac{1}{\gamma}\Big[\frac{2}{3}\Big]= 0.5
\end{eqnarray*}

We thus propose to use $\delta^\mathit{DctRev}$ in general for probabilistic belief revision:
\begin{equation*}
\delta^\mathit{DctRev}(\alpha,w,w'):=\left\lbrace
\begin{array}{ll}
\delta^\mathit{BC}(w,w') & \textit{if } b(\alpha)>0\\
\delta^{=0}(\alpha,w,w') & \textit{otherwise}
\end{array}
\right.
\end{equation*}

The following definition is derived from the one above.
\begin{definition}
\begin{equation*}
\delta^\mathit{DctRev}(\alpha,w,w'):=\left\lbrace
\begin{array}{ll}
\delta^\mathit{BC}(w,w') & \textit{if } b(\alpha)>0\\
\delta^{=0}(\alpha,w,w') & \textit{otherwise}
\end{array}
\right.
\end{equation*}
where $\delta$ is a non-evidence (n-e-) relaxed and an inverse-distance weight function.
\end{definition}
The reason why $\delta$ is n-e-relaxed is similar to the reason why $\delta$  of $\delta^\mathit{ClsRev}$ is. Recall that $b(\mathit{book})\neq0$.
Then
\begin{eqnarray*}
b^{\mathsf{EDI}^\mathit{DctRev}}_\mathit{book}(11) &=& \frac{1}{\gamma}\Big[b(10)\delta^\mathit{DctRev}(\alpha,11,10) +\\
&& b(01)\delta^\mathit{DctRev}(\alpha,11,01)\Big]\\
&=& \frac{1}{\gamma}\Big[0.5\times0 + 0.5\times0\Big]= 0
\end{eqnarray*}
\begin{eqnarray*}
b^{\mathsf{EDI}^\mathit{DctRev}}_\mathit{book}(10) &=& \frac{1}{\gamma}\Big[b(10)\delta^\mathit{DctRev}(\alpha,10,10) +\\
&& b(01)\delta^\mathit{DctRev}(\alpha,10,01)\Big]\\
&=& \frac{1}{\gamma}\Big[0.5\times1 + 0.5\times0\Big]= 1
\end{eqnarray*}

And $\mathsf{EDI}^\mathit{DctRev}$ can also be used in cases where $b(\alpha)=0$:
Let $b=\langle0.3,0.7,0,0\rangle$ and let $\alpha=\lnot\mathit{book}$.
Observe that $\alpha$ contradicts the agent's current belief state $b$. Let
\begin{equation*}
\delta(\alpha,w,w'):=\left\lbrace
\begin{array}{ll}
0 & \textit{if }w,w'\Vdash\alpha \textit{ and } w\neq w'\\
\delta^\mathit{dfr}(\alpha,w,w') & \textit{otherwise},
\end{array}
\right.
\end{equation*}
with $\eta=0.1$. Note that $\delta$ is an n-e-relaxed and retentive inverse-distance weight function.
\begin{eqnarray*}
b^{\mathsf{EDI}^\mathit{DctRev}}_{\lnot\mathit{book}}(01) &=& \frac{1}{\gamma}\Big[b(11)\delta^\mathit{DctRev}(\alpha,01,11) +\\
&& b(10)\delta^\mathit{DctRev}(\alpha,01,10)\Big]\\
&=& \frac{1}{\gamma}\Big[0.3\delta^\mathit{dfr}(01,11) + 0.7\delta^\mathit{dfr}(01,10)\Big]\\
&=& \frac{1}{\gamma}\Big[0.3\frac{2.1-1}{2.1} + 0.7\frac{2.1-2}{2.1}\Big]\\
&=& \frac{1}{\gamma}\Big[0.16 + 0.03\Big]= 0.\bar{3}\\
\end{eqnarray*}
\begin{eqnarray*}
b^{\mathsf{EDI}^\mathit{DctRev}}_{\lnot\mathit{book}}(00) &=& \frac{1}{\gamma}\Big[b(11)\delta^\mathit{DctRev}(\alpha,00,11) +\\
&& b(10)\delta^\mathit{DctRev}(\alpha,00,10)\Big]\\
&=& \frac{1}{\gamma}\Big[0.3\delta^\mathit{dfr}(00,11) + 0.7\delta^\mathit{dfr}(00,10)\Big]\\
&=& \frac{1}{\gamma}\Big[0.3\frac{2.1-2}{2.1} + 0.7\frac{2.1-1}{2.1}\Big]\\
&=& \frac{1}{\gamma}\Big[0.01 + 0.37\Big]= 0.\bar{6}\\
\end{eqnarray*}
Initially, the agent is certain that the book is on the floor and there is a relatively high likelihood (0.7) that the magazine is on the table.
Then the agent hears from a reliable source that actually the book is definitely on the table. The agent revises its beliefs accordingly, and reasonably still believes to a relatively high degree ($0.\bar{6}$) that the magazine is on the table.

\begin{proposition}
$\delta^\mathit{DctRev}$ is retentive.
\end{proposition}
\begin{proof}
Assume $w,w'\Vdash\alpha$ and $w\neq w'$.
Assume $b(\alpha)=0$.
Then only the second line of the definition of $\delta^\mathit{DctRev}$ is applicable.
That is, $\delta^\mathit{DctRev}=0$.
Now assume $b(\alpha)\neq0$.
Again, only the second line is applicable, making $\delta^\mathit{DctRev}=0$.
The retention property is thus satisfied.
\end{proof}

\begin{proposition}
$\delta^\mathit{DctRev}$ satisfies non-negativity, identity and symmetry.
\end{proposition}
\begin{proof}
Non-negativity follows directly.
If $w=w'$, then only the first and third lines are applicable. In both cases, identity holds.

The first line of the definition is satisfied for $\delta^\mathit{DctRev}(\alpha,w,w')$ iff it is satisfied for $\delta^\mathit{DctRev}(\alpha,w',w)$.
Similarly for the second line.
Because $\delta$ is assumed to be an inverse-distance weight function, $\delta$ is symmetrical. It thus follows that $\delta^\mathit{DctRev}$ satisfies symmetry.
\end{proof}

\begin{proposition}
$\delta^\mathit{DctRev}$ does not satisfy weak inversity, strict inversity nor equi-distance.
\end{proposition}
\begin{proof}
For all $w,w',w'',w'''\in W$:

Assume $d(w,w')=d(w'',w''')$.
If $d(w,w')=d(w'',w''')\neq0$, then, by the identity condition of $d$, $w\neq w'$ and $w''\neq w'''$, and only the second and third lines are applicable.
Assume ``$w'',w'''\Vdash\alpha \textit{ if } b(\alpha)=0$'' holds, but that ``$w,w'\Vdash\alpha \textit{ if } b(\alpha)=0$'' does not hold.
Then $\delta^\mathit{DctRev}(\alpha,w,w') = \delta(\alpha,w,w')$ and $\delta^\mathit{DctRev}(\alpha,w'',w''') = 0$.
Now, it could be that $\delta(\alpha,w,w')>0$.
Hence, if $d(w,w')\geq d(w'',w''')$ it is not in general the case that $\delta^\mathit{DctRev}(\alpha,w,w')\leq \delta^\mathit{DctRev}(\alpha,w'',w''')$.
Thus, weak inversity fails. This case also shows that equi-distance and strict inversity fail.
\end{proof}

\begin{proposition}
$\delta^\mathit{DctRev}$ is faithful iff $\delta$ is.
\end{proposition}
\begin{proof}
Assume $w\neq w'$.

Assume $\delta$ is not faithful.
Assume $b(\alpha)=0$ and $w'\not\Vdash\alpha$. Then only the third line is applicable. It could happen that $\delta(\alpha,w,w')=1$, making $\delta^\mathit{DctRev}(\alpha,w,w')=1$, making $\delta^\mathit{DctRev}$ unfaithful.

Now assume $\delta$ is faithful.
Only the second and third lines are applicable.
In both cases,\\ $\delta^\mathit{DctRev}(\alpha,w,w')$ $<$ $1$ (particularly, $\delta(\alpha,w,w')<1$, because $w\neq w'$).
\end{proof}

\subsection{Probabilistic Update via Classical Update}

This approach is the same as we took for probabilistic revision via classical revision; to determine the new knowledge base exactly as one would in classical belief update, and then to assign probabilities to the newly believed worlds.

Recall that $W^b := \{w\in W\mid (w,p)\in b, p\neq 0\}$ and $\psi^b$ is a sentence modeled by all worlds in $W^b$ and no others.
Let $\delta^\mathit{ClsUpd}$ be defined as follows.
\begin{equation*}
\delta^\mathit{ClsUpd}(\alpha,w,w'):=\left\lbrace
\begin{array}{ll}
1 & \textit{if }w=w'\\
\delta(\alpha,w,w') & \textit{if }w\in\mathit{Mod}(\psi^b\diamond\alpha)\\
0 & \textit{otherwise},
\end{array}
\right.
\end{equation*}
where $\diamond$ is some (acceptable) update operator and $\delta$ is a relaxed inverse-distance weight function.

\begin{proposition}
If $\mathit{Mod}(\psi^b\diamond\alpha) = \mathit{Mod}(\alpha)$, then $\delta^\mathit{ClsUpd}$ is relaxed.
\end{proposition}
\begin{proof}
Assume $\mathit{Mod}(\psi^b\diamond\alpha) = \mathit{Mod}(\alpha)$.
Assume $w\Vdash\alpha$.
If $w=w'$, then only the first line of the definition of $\delta^\mathit{ClsUpd}$ is applicable, and $\delta^\mathit{ClsUpd}(\alpha,w,w')\neq0$.
If $w\neq w'$, then only the second and third lines are applicable.
But actually, by the two assumptions, only the second line is applicable.
This implies that for all $w,w'\in W$, if $w\Vdash\alpha$, then $\delta^\mathit{ClsUpd}(\alpha,w,w')\neq0$.
Therefore, $\delta^\mathit{ClsUpd}$ is relaxed.
\end{proof}

\begin{proposition}
$\delta^\mathit{ClsUpd}$ satisfies non-negativity, identity and symmetry.
\end{proposition}
\begin{proof}
Non-negativity and identity follow directly from the definition of $\delta^\mathit{ClsUpd}$.
We look at the three cases (lines) which make up the definition of $\delta^\mathit{ClsUpd}$.
If $w=w'$, then $\delta^\mathit{ClsUpd}(\alpha,w,w')=\delta^\mathit{ClsRev}(\alpha,w',w)=1$.
If $w\in\mathit{Mod}(\psi^b\diamond\alpha)$, then $\delta^\mathit{ClsUpd}(\alpha,w,w')=\delta(\alpha,w,w')$. But $\delta(\alpha,w,w')$ is assumed to be an inverse-distance weight function, which implies that $\delta(\alpha,w,w')$ is symmetrical.
For all other cases (third line), it must be that $\delta^\mathit{ClsUpd}(\alpha,w,w')=\delta^\mathit{ClsUpd}(\alpha,w',w)=0$.
\end{proof}

\begin{proposition}
$\delta^\mathit{ClsUpd}$ does not satisfy weak inversity, strong inversity nor equi-distance.
\end{proposition}
\begin{proof}
Suppose the vocabulary has three atoms, the distance function is Hamming and $\delta$ is $\delta^\mathit{rcp}$ with $\eta=1$.
Assume $\mathit{Mod}(\psi^b) = \{101,001,000\}$ and $111,110,011,010\Vdash\alpha$ and all other worlds do not model $\alpha$.
Finally, assume that $\diamond$ is defined such that $\mathit{Mod}(\psi^b\diamond\alpha)=\{111,011,010\}$.
Note that $\mathit{Mod}(\psi^b\diamond\alpha)\subseteq\mathit{Mod}(\alpha)$.

Then $d(010,000)=d(110,010)=1$,
$\delta^\mathit{ClsUpd}(\alpha,010,000)=\delta^\mathit{rcp}(\alpha,010,000) = 1/2$ and\\ $\delta^\mathit{ClsUpd}(\alpha,110,010) = 0$.
Therefore, $d(010,000)\geq d(110,010)$, but $\delta^\mathit{ClsUpd}(\alpha,010,000)\not\leq$\\ $ \delta^\mathit{ClsUpd}(\alpha,110,010)$. Hence, weak inversity fails for $\delta^\mathit{ClsUpd}$. This case also proves that equi-distance fails.

Due to strict inversity implying weak inversity, by contraposition, strict inversity fails.
\end{proof}

\begin{corollary}
Because $\delta^\mathit{ClsUpd}$ is not weakly inverse, it is not an inverse-distance weight function.
\end{corollary}

\begin{proposition}
$\delta^\mathit{ClsUpd}$ is faithful iff $\delta$ is.
\end{proposition}
\begin{proof}
Assume $w\neq w'$.

Assume $\delta$ is not faithful.
Next, assume that $w\in\mathit{Mod}(\psi^b\diamond\alpha)$. It could thus happen that $\delta(\alpha,w,w')$ $=$ $1$, implying that $\delta^\mathit{ClsUpd}$ is not faithful.

Now assume $\delta$ is faithful.
Only the second and third lines are applicable.
In both cases,\\ $\delta^\mathit{ClsUpd}(\alpha,w,w')$ $<$ $1$ (particularly, $\delta(\alpha,w,w')<1$, because $w\neq w'$).
\end{proof}

\subsection{Probabilistic Update with EDI Directly}

Finally, we propose to use any relaxed (e-relaxed and n-e-relaxed) inverse-distance weight function for update. That is, we propose that one may use operation $\mathsf{EDI}^\mathit{DctUpd}$ for updating, such that it is relaxed and where $\delta^\mathit{DctUpd}$ is an inverse-distance weight function.

Recall the example in Section~\ref{sec:Probabilistic-Revision-with-EDI-Directly}, where $b=\langle0.3,0.7,0,0\rangle$, $\alpha=\lnot\mathit{book}$, $\delta^\mathit{DctRev}$ is instantiated as $\delta^\mathit{dfr}$ and $\eta=0.1$.
Notice that if an agent were to \textit{update} its beliefs ($\langle0.3,0.7,0,0\rangle$) with $\lnot\mathit{book}$, the resulting belief state would be exactly the same as for revision via $\mathsf{EDI}^\mathit{DctRev}$: $\langle0,0,0.\bar{3},0.\bar{6}\rangle$.
%
The reader may confirm that on the second application of $\mathsf{EDI}^\mathit{DctRev}$, the agent still believes $\langle0,0,0.\bar{3},0.\bar{6}\rangle$. On the second application of $\mathsf{EDI}^\mathit{DctUpd}$, however, the agent believes $\langle0,0,0.45,0.55\rangle$.

Each of strict inversity, equi-distance and faithfulness might be satisfied, depending on the particular instantiation of $\delta^\mathit{DctUpd}$.

\section{Rationality Postulates for EDI}

In this section, we present and propose \textit{rationality postulates} -- conditions which should be satisfied -- for belief revision and belief update. We do not claim that these are sufficient. Some of them are necessary, but some may not be. Nonetheless, the postulates have been given due consideration.
For this study, we identify three core (necessary) rationality postulates for revision and update.
We prove that $\mathsf{EDI}^\mathit{ClsRev}$ and $\mathsf{EDI}^\mathit{DctRev}$ satisfy the three core revision postulates, and that $\mathsf{EDI}^\mathit{ClsUpd}$ and $\mathsf{EDI}^\mathit{DctUpd}$ satisfy the three core update postulates.

\subsection{Probabilistic Revision Postulates}
We adopt the rationality postulates for probabilistic belief revision from \citet{g88} and we employ probabilistic versions of the rationality postulates for (non-probabilistic) belief update from \citet{km92}.
%

%
%
First, we discuss the operation called  \emph{expansion}, because it is mentioned in the postulates below.
Conventionally, (classical) expansion (denoted +) is the logical consequences of $K\uni\{\alpha\}$, where $\alpha$ is new information and $K$ is the current belief set.
Or if the current beliefs can be captured as a single sentence $\beta$, expansion is defined simply as $\beta+\alpha \equiv \beta\land\alpha$.
We denote the expansion of belief state $b$ with $\alpha$ as $b^+_\alpha$.

(Unless stated otherwise, it is assumed that $\alpha$ is logically satisfiable.)
The probabilistic belief revision postulates (adapted from \citep{g88}) in our notation are
\begin{itemize}
\itemsep=0pt
\item[]($P^\circ1$) $b^\circ_\alpha$ is a belief state
\item[]($P^\circ2$) $b^\circ_\alpha(\alpha)=1$
\item[]($P^\circ3$) If $\alpha\equiv\beta$, then $b^\circ_\alpha=b^\circ_\beta$
\item[]($P^\circ4$) If $b(\alpha)>0$, then $b^\circ_\alpha = b^+_\alpha$
\item[]($P^\circ5$) If $b^\circ_\alpha(\beta)>0$, then $b^\circ_{\alpha\land\beta}=(b^\circ_\alpha)^+_\beta$
\end{itemize}
We take $P^\circ1$ - $P^\circ3$ to be self explanatory, and to be the three core postulates.
$P^\circ4$ is an interpretation of the AGM postulate which says that if the evidence is consistent with the currently held beliefs, then revision amounts to expansion.

$P^\circ5$ says that if $\beta$ is deemed possible
in the belief state revised with $\alpha$, then expanding the revised belief state with $\beta$ should be equal to revising the original belief state with the conjunction of $\alpha$ and $\beta$.

We propose adding this postulate:
\begin{itemize}
\item[]($P^\circ6$) If $\beta\models\alpha$, then $b^\circ_\alpha(\beta)\geq b(\beta)$
\end{itemize}
$P^\circ6$ says that the belief in an $\alpha$-world cannot decrease due to the reception of information that $\alpha$.

\bigskip
Recall that
\begin{equation*}
\delta^\mathit{ClsRev}(\alpha,w,w'):=\left\lbrace
\begin{array}{ll}
1 & \textit{if }w=w'\\
\delta(\alpha,w,w') & \textit{if }w\in\mathit{Mod}(\psi^b\circ\alpha)\\
0 & \textit{otherwise},
\end{array}
\right.
\end{equation*}
where $\circ$ is some (acceptable) revision operator and $\delta$ is a non-evidence relaxed and retentive inverse-distance weight function.

\begin{proposition}
($P^\circ1$) is satisfied for $\mathsf{EDI}^\mathit{ClsRev}$.
\end{proposition}
\begin{proof}
Due to normalization (via $\gamma$) in the definition of EDI (Def.~\ref{def:EDI}), $b^\circ_\alpha$ is a belief state whenever there exists at least one world $w\in W$ s.t.\ if $w\Vdash\alpha$, then $\sum_{w'\in W}b(w')\delta^\mathit{ClsRev}(\alpha,w,w')>0$.
Assume that $w\Vdash\alpha$.
It must be shown that there exists at least one world $w'$ for which $b(w')>0$ and $\delta^\mathit{ClsRev}(\alpha,w,w')>0$.

Proof by contradiction:
Let $w$ be an arbitrary world in $W$.
(Main assumption) Assume there exists no world $w'$ for which $b(w')>0$ and $\delta^\mathit{ClsRev}(\alpha,w,w')>0$.
Assume $w=w'$.
Then the first line of the definition of $\delta^\mathit{ClsRev}$ is applicable and it must be that $b(w')=0$. But this is impossible because it implies that for all worlds $w$, $b(w)=0$ ($b$ is implicitly assumed to be well-defined).
Therefore, it must be that $w\neq w'$.

Let $w\in\mathit{Mod}(\psi^b\circ\alpha)$ (there must exists at least one such $w$).
Then the second line is applicable. Let $b(w')>0$ (there must exist at least on such $w'$). 
This implies that $\delta^\mathit{ClsRev}(\alpha,w,w')=0$, which implies that $\delta(\alpha,w,w')=0$.
Now, either $w'\Vdash\alpha$ or $w'\not\Vdash\alpha$.

Assume $w'\Vdash\alpha$.
Recall that $b(w')>0$.
Note that the first line will eventually become applicable, making $\delta^\mathit{ClsRev}(\alpha,w',w')=1$ and contradicting the main assumption.
Therefore, it must be the case that $w'\not\Vdash\alpha$. Recall that $\delta$ is n-e-relaxed, that is, $\forall w,w'\in W$, if $w\Vdash\alpha$ and $w'\not\Vdash\alpha$, then $\delta(\alpha,w,w')\neq0$. This also contradicts the main assumption.

There is no other way to satisfy the main assumption. It must thus be the case that there exists some world $w'$ for which $b(w')>0$ and $\delta^\mathit{ClsRev}(\alpha,w,w')>0$, implying that there exists at least one world $w\in W$ s.t.\ if $w\Vdash\alpha$, then $\sum_{w'\in W}b(w')\delta^\mathit{ClsRev}(\alpha,w,w')>0$.
\end{proof}

\begin{proposition}
($P^\circ2$) is satisfied for $\mathsf{EDI}^\mathit{ClsRev}$.
\end{proposition}
\begin{proof}
The proposition is satisfied when $b^{\mathsf{EDI}^\mathit{ClsRev}}_\alpha(\alpha)=1$.
It is almost a direct consequence of the definition of $\mathsf{EDI}$:
$\forall w\in W$, $b^{\mathsf{EDI}^\mathit{ClsRev}}_\alpha(w)>0$ only if $w\Vdash\alpha$.
\end{proof}

\begin{proposition}
($P^\circ3$) is satisfied for $\mathsf{EDI}^\mathit{ClsRev}$.
\end{proposition}
\begin{proof}
If $\alpha\equiv\beta$, then $\forall w\in W$, $w\Vdash\alpha$ iff $w\Vdash\beta$ and $b^{\mathsf{EDI}^\mathit{ClsRev}}_\alpha(w)=b^{\mathsf{EDI}^\mathit{ClsRev}}_\beta(w)$, implying that  $b^{\mathsf{EDI}^\mathit{ClsRev}}_\alpha=b^{\mathsf{EDI}^\mathit{ClsRev}}_\beta$.
\end{proof}

\bigskip
Recall that
\begin{equation*}
\delta^\mathit{DctRev}(\alpha,w,w'):=\left\lbrace
\begin{array}{ll}
\delta^\mathit{BC}(w,w') & \textit{if } b(\alpha)>0\\
\delta^{=0}(\alpha,w,w') & \textit{otherwise}
\end{array}
\right.
\end{equation*}
where $\delta$ is a non-evidence relaxed and an inverse-distance weight function.

\begin{proposition}
($P^\circ1$) is satisfied for $\mathsf{EDI}^\mathit{DctRev}$.
\end{proposition}
\begin{proof}
Due to normalization (via $\gamma$) in the definition of EDI (Def.~\ref{def:EDI}), $b^\circ_\alpha$ is a belief state whenever there exists at least one world $w\in W$ s.t.\ if $w\Vdash\alpha$, then $\sum_{w'\in W}b(w')\delta^\mathit{DctRev}(\alpha,w,w')>0$.
Assume that $w\Vdash\alpha$.
It must be shown that there exists at least one world $w'$ for which $b(w')>0$ and $\delta^\mathit{DctRev}(\alpha,w,w')>0$.

Proof by contradiction:
Let $w$ be an arbitrary world in $W$.
(Main assumption) Assume there exists no world $w'$ for which $b(w')>0$ and $\delta^\mathit{DctRev}(\alpha,w,w')>0$.

Either $b(\alpha)=0$ or $b(\alpha)>0$.
Assume $b(\alpha)>0$.
Then there exists at least one $\alpha$-world with probability greater than zero. Let $w$ be that world. Thus, when $w=w'$, $\delta^\mathit{DctRev}(\alpha,w,w')=1$ (first line of definition) and the main assumption fails in this case.
Assume $b(\alpha)=0$.
Then there exists at least one non-$\alpha$-world with probability greater than zero. Let $w'$ be that world. Hence, the first and second lines are inapplicable, because ``$w,w'\Vdash\alpha \textit{ if } b(\alpha)=0$'' fails.
Therefore, only the third line is applicable.
$\delta$ is n-e-relaxed, so $\forall w,w'\in W$, if $w\Vdash\alpha$ and $w'\not\Vdash\alpha$, then $\delta(\alpha,w,w')\neq0$. By the main assumption, this implies that for all worlds $w'$, $b(w)=0$. But this contradicts our deduction that there exists some $w'$ for which $b(w')>0$.

There is no other way to satisfy the main assumption. It must thus be the case that there exists some world $w'$ for which $b(w')>0$ and $\delta^\mathit{DctRev}(\alpha,w,w')>0$, implying that there exists at least one world $w\in W$ s.t.\ if $w\Vdash\alpha$, then $\sum_{w'\in W}b(w')\delta^\mathit{DctRev}(\alpha,w,w')>0$.
\end{proof}

\begin{proposition}
($P^\circ2$) is satisfied for $\mathsf{EDI}^\mathit{DctRev}$.
\end{proposition}
\begin{proof}
The proposition is satisfied when $b^{\mathsf{EDI}^\mathit{DctRev}}_\alpha(\alpha)=1$.
It is almost a direct consequence of the definition of $\mathsf{EDI}$:
$\forall w\in W$, $b^{\mathsf{EDI}^\mathit{DctRev}}_\alpha(w)>0$ only if $w\Vdash\alpha$.
\end{proof}

\begin{proposition}
($P^\circ3$) is satisfied for $\mathsf{EDI}^\mathit{DctRev}$.
\end{proposition}
\begin{proof}
If $\alpha\equiv\beta$, then $\forall w\in W$, $w\Vdash\alpha$ iff $w\Vdash\beta$ and $b^{\mathsf{EDI}^\mathit{DctRev}}_\alpha(w)=b^{\mathsf{EDI}^\mathit{DctRev}}_\beta(w)$, implying that  $b^{\mathsf{EDI}^\mathit{DctRev}}_\alpha=b^{\mathsf{EDI}^\mathit{DctRev}}_\beta$.
\end{proof}

\subsection{Probabilistic Update Postulates}
A set of probabilistic belief update postulates is now adapted from \citep{km92}'s classical postulates.
Each classical postulate is translated into probabilistic counterpart (in our notation).
\begin{itemize}
\itemsep=0pt
\item[]($U1$) $\beta\diamond\alpha\models \alpha$
\item[]($P^\diamond1$) $b^\diamond_\alpha(\alpha)=1$
\end{itemize}
\begin{itemize}
\itemsep=0pt
\item[]($U2$) if $\beta\models \alpha$, then $\beta\diamond\alpha\equiv \beta$
\item[]($P^\diamond2a$) if $b(\alpha)=1$, then $b^\diamond_\alpha = b$
\item[]($P^\diamond2b$) if $\phi\models\alpha$, $b(\phi)>0$ iff $b^\diamond_\alpha(\phi)>0$
\item[]($P^\diamond2c$) if $\phi\models\alpha$, then if $b(\phi)>0$, then $b^\diamond_\alpha(\phi)>0$
\end{itemize}
It is arguable which of $P^\diamond2a$, $P^\diamond2b$ or $P^\diamond2c$ are appropriate translations of $U2$ (if any).
$P^\diamond2a$ is perhaps too strong and $P^\diamond2c$ perhaps too weak. If $P^\diamond2b$ is taken to be an appropriate translation, then we argue that in a probabilistic setting, it is too strong. For instance, it implies that if $b(\phi)=0$, then $b^\diamond_\alpha(\phi)=0$, but this is not the case when $\diamond$ is e-relaxed.
\begin{itemize}
\itemsep=0pt
\item[]($U3$) if both $\beta$ and $\alpha$ are satisfiable, then $\beta\diamond\alpha$ is also satisfiable
\item[]($P^\diamond3$) $b^\diamond_\alpha$ is a belief state
\end{itemize}
$P^\diamond3$ has the simpler form because $\alpha$ is assumed satisfiable and $b$ is assumed to be a belief state.
\begin{itemize}
\itemsep=0pt
\item[]($U4$) if $\beta_1\equiv\beta_2$ and $\alpha_1\equiv\alpha_2$, then $\beta_1\diamond\alpha_1 \equiv \beta_2\diamond\alpha_2$ 
\item[]($P^\diamond4$) if $\alpha_1\equiv\alpha_2$, then $b^\diamond_{\alpha_1}=b^\diamond_{\alpha_2}$
\end{itemize}
\begin{itemize}
\itemsep=0pt
\item[]($U5$) $(\beta\diamond\alpha)\land\phi\models \beta\diamond(\alpha\land\phi)$
\item[]($P^\diamond5$) if $\alpha\land\phi$ is satisfiable and $\psi\models\phi$, then:\\ $b^\diamond_{\alpha\land\phi}(\psi)\geq b^\diamond_\alpha(\psi)$
\end{itemize}
Suppose $\kappa$ is a knowledge base and $\kappa\land\alpha'$ implies $\kappa\land\beta'$.
Then in terms of probabilities, one should expect $b(\beta')\geq b(\alpha')$, where $b$ is the `knowledge base'. Given this metaphor, one can derive from $U5$ the translation `if $\alpha\land\phi$ is satisfiable, then for all $w \in W$, if $w\Vdash\phi$, then $b^\diamond_{\alpha\land\phi}(w)\geq b^\diamond_\alpha(w)$, of which $P^\diamond5$ is the sentential version.
\begin{itemize}
\itemsep=0pt
\item[]($U6$) if $\beta\diamond\alpha_1\models \alpha_2$ and $\beta\diamond\alpha_2\models \alpha_1$, then:\\ $\beta\diamond\alpha_1 \equiv \beta\diamond\alpha_2$
\item[]($P^\diamond6a$) if $b^\diamond_{\alpha_1}(\alpha_2)=1$ and $b^\diamond_{\alpha_2}(\alpha_1)=1$, then:\\ $b^\diamond_{\alpha_1} = b^\diamond_{\alpha_2}$
\item[]($P^\diamond6b$) if $b^\diamond_{\alpha_1}(\alpha_2)=1$ and $b^\diamond_{\alpha_2}(\alpha_1)=1$, then:\\ $b^\diamond_{\alpha_1}(\phi)>0 \iff b^\diamond_{\alpha_2}(\phi)>0$
\end{itemize}
$P^\diamond6b$ is weaker than $P^\diamond6a$ but arguably more reasonable/rational in a probabilistic setting. 
\begin{itemize}
\itemsep=0pt
\item[]($U7$) if $\beta$ is complete, then:\\ $\beta\diamond\alpha_1 \land \beta\diamond\alpha_2\models \beta\diamond(\alpha_1 \lor \alpha_2)$ 
\item[]($P^\diamond7$) if there exists a $w$ s.t.\ $b(w)=1$, then:\\ $\min\{b^\diamond_{\alpha_1}(\phi),b^\diamond_{\alpha_2}(\phi)\}\leq b^\diamond_{\alpha_1\lor\alpha_2}(\phi)\leq b^\diamond_{\alpha_1}(\phi)+b^\diamond_{\alpha_2}(\phi)$
\end{itemize}
KM's justification for $U7$ is ``If some possible world results from updating a complete KB with $\mu_1$ and it also results from updating it with $\mu_2$, then this possible world must also result from updating the KB with $\mu_1\lor\mu_2$''.
However, we go a step farther (given a probabilistic setting), and say that the belief in any $\beta$ given complete belief state resulting from updating with $\mu_1\lor\mu_2$ should be no less than the least of the beliefs in $\beta$ after updating the belief state with $\mu_1$ and updating it with $\mu_2$ separately, and no more than the sum of the belief in $\beta$ after updating the belief state with $\mu_1$ and updating it with $\mu_2$ separately.
\begin{itemize}
\itemsep=0pt
\item[]($U8$) $(\beta_1\lor\beta_2)\diamond\alpha \equiv (\beta_1\diamond\alpha)\lor(\beta_2\diamond\alpha)$
\end{itemize}
At this time, we do not have a satisfactory translation for $U8$.

$P^\diamond1$, $P^\diamond3$ and $P^\diamond4$ are taken to be the three core postulates.

\bigskip
Recall that
\begin{equation*}
\delta^\mathit{ClsUpd}(\alpha,w,w'):=\left\lbrace
\begin{array}{ll}
1 & \textit{if }w=w'\\
\delta(\alpha,w,w') & \textit{if }w\in\mathit{Mod}(\psi^b\diamond\alpha)\\
0 & \textit{otherwise},
\end{array}
\right.
\end{equation*}
where $\diamond$ is some (acceptable) update operator and $\delta$ is a relaxed inverse-distance weight function.

\begin{proposition}
($P^\diamond1$) is satisfied for $\mathsf{EDI}^\mathit{ClsUpd}$.
\end{proposition}
\begin{proof}
The proposition is satisfied when $b^{\mathsf{EDI}^\mathit{ClsUpd}}_\alpha(\alpha)=1$.
It is almost a direct consequence of the definition of $\mathsf{EDI}$:
$\forall w\in W$, $b^{\mathsf{EDI}^\mathit{ClsUpd}}_\alpha(w)>0$ only if $w\Vdash\alpha$.
\end{proof}

\begin{proposition}
($P^\diamond3$) is satisfied for $\mathsf{EDI}^\mathit{ClsUpd}$.
\end{proposition}
\begin{proof}
We must prove that $b^{\mathsf{EDI}^\mathit{ClsUpd}}_\alpha$ is a belief state.
Due to normalization (via $\gamma$) in the definition of EDI (Def.~\ref{def:EDI}), $b^{\mathsf{EDI}^\mathit{ClsUpd}}_\alpha$ is a belief state whenever there exists at least one world $w\in W$ s.t.\ if $w\Vdash\alpha$, then $\sum_{w'\in W}b(w')\delta^\mathit{ClsUpd}(\alpha,w,w')>0$.
Assume that $w\Vdash\alpha$.
It must be shown that there exists at least one world $w'$ for which $b(w')>0$ and $\delta^\mathit{ClsUpd}(\alpha,w,w')>0$.

Proof by contradiction:
Let $w$ be an arbitrary world in $W$.
(Main assumption) Assume there exists no world $w'$ for which $b(w')>0$ and $\delta^\mathit{ClsUpd}(\alpha,w,w')>0$.
Assume $w=w'$.
Then the first line of the definition of $\delta^\mathit{ClsUpd}$ is applicable and it must be that $b(w')=0$. But this is impossible because it implies that for all worlds $w$, $b(w)=0$ ($b$ is implicitly assumed to be well-defined).
Therefore, it must be that $w\neq w'$.

Let $w\in\mathit{Mod}(\psi^b\diamond\alpha)$ (there must exists at least one such $w$).
Then the second line is applicable. Let $b(w')>0$ (there must exist at least on such $w'$). 
This implies that $\delta^\mathit{ClsUpd}(\alpha,w,w')=0$, which implies that $\delta(\alpha,w,w')=0$.
Now, either $w'\Vdash\alpha$ or $w'\not\Vdash\alpha$.

Assume $w'\Vdash\alpha$.
Recall that $b(w')>0$.
Note that the first line will eventually become applicable, making $\delta^\mathit{ClsUpd}(\alpha,w',w')=1$ and contradicting the main assumption.
Therefore, it must be the case that $w'\not\Vdash\alpha$. Recall that $\delta$ is relaxed and thus n-e-relaxed -- that is, $\forall w,w'\in W$, if $w\Vdash\alpha$ and $w'\not\Vdash\alpha$, then $\delta(\alpha,w,w')\neq0$. This also contradicts the main assumption.

There is no other way to satisfy the main assumption. It must thus be the case that there exists some world $w'$ for which $b(w')>0$ and $\delta^\mathit{ClsUpd}(\alpha,w,w')>0$, implying that there exists at least one world $w\in W$ s.t.\ if $w\Vdash\alpha$, then $\sum_{w'\in W}b(w')\delta^\mathit{ClsUpd}(\alpha,w,w')>0$.
\end{proof}

\begin{proposition}
($P^\diamond4$) is satisfied for $\mathsf{EDI}^\mathit{ClsUpd}$.
\end{proposition}
\begin{proof}
If $\alpha\equiv\beta$, then $\forall w\in W$, $w\Vdash\alpha$ iff $w\Vdash\beta$ and $b^{\mathsf{EDI}^\mathit{ClsUpd}}_\alpha(w)=b^{\mathsf{EDI}^\mathit{ClsUpd}}_\beta(w)$, implying that $b^{\mathsf{EDI}^\mathit{ClsUpd}}_\alpha=b^{\mathsf{EDI}^\mathit{ClsUpd}}_\beta$.
\end{proof}

\bigskip
Recall that $\delta^\mathit{DctUpd}$ is a relaxed inverse-distance weight function.

\begin{proposition}
($P^\diamond1$) is satisfied for $\mathsf{EDI}^\mathit{DctUpd}$.
\end{proposition}
\begin{proof}
The proposition is satisfied when $b^{\mathsf{EDI}^\mathit{DctUpd}}_\alpha(\alpha)=1$.
It is almost a direct consequence of the definition of $\mathsf{EDI}$:
$\forall w\in W$, $b^{\mathsf{EDI}^\mathit{DctUpd}}_\alpha(w)>0$ only if $w\Vdash\alpha$.
\end{proof}

\begin{proposition}
($P^\diamond3$) is satisfied for $\mathsf{EDI}^\mathit{DctUpd}$.
\end{proposition}
\begin{proof}
We must prove that $b^{\mathsf{EDI}^\mathit{DctUpd}}_\alpha$ is a belief state.
Due to normalization (via $\gamma$) in the definition of EDI (Def.~\ref{def:EDI}), $b^{\mathsf{EDI}^\mathit{DctUpd}}_\alpha$ is a belief state whenever there exists at least one world $w\in W$ s.t.\ if $w\Vdash\alpha$, then $\sum_{w'\in W}b(w')\delta^\mathit{DctUpd}(\alpha,w,w')>0$.
Assume that $w\Vdash\alpha$.
It must be shown that there exists at least one world $w'$ for which $b(w')>0$ and $\delta^\mathit{DctUpd}(\alpha,w,w')>0$.

Let $b(w')>0$ (there must exist at least one such). By relaxation, $\delta^\mathit{DctUpd}(\alpha,w,w')>0$.
\end{proof}

\begin{proposition}
($P^\diamond4$) is satisfied for $\mathsf{EDI}^\mathit{DctUpd}$.
\end{proposition}
\begin{proof}
If $\alpha\equiv\beta$, then $\forall w\in W$, $w\Vdash\alpha$ iff $w\Vdash\beta$ and $b^{\mathsf{EDI}^\mathit{DctUpd}}_\alpha(w)=b^{\mathsf{EDI}^\mathit{DctUpd}}_\beta(w)$, implying that $b^{\mathsf{EDI}^\mathit{DctUpd}}_\alpha=b^{\mathsf{EDI}^\mathit{DctUpd}}_\beta$.
\end{proof}

\section{Conclusion}

We have proposed probabilistic belief revision and belief update operators, both within an imaging framework. The role of a weight function -- usually expected to be inversely proportional to the 'distance' between worlds -- is central to the Expected Distance Imaging (EDI) framework.
In this paper, we have only begun the investigation into the behavior or properties of various instantiations of a general form of imaging. There is still much work to be done. For instance, what is the effect of retention versus relaxation and when is one more appropriate than the other?  We have worked with the hypothesis that retention relates to revision and that relaxation relates to update, but to what degree is this hypothesis true?
What is the meaning of the effect of the size of $\eta$ in $\delta^\mathit{rcp}$ and $\delta^\mathit{dfr}$?

We would like to make a deeper study of the probabilistic rationality postulates, especially those for update (as they have not been given much attention in the literature). We would then like to test various EDI instantiations against the postulates and use the results to perhaps design new EDI operators.
We hope that these insights might lead to further insights in order to better understand the relationship between revision and update, whether classical or not.

\newpage
\bibliographystyle{apalike}
\bibliography{references}

\end{document}